\documentclass[11pt,a4paper]{scrartcl}
\usepackage[top=3cm, bottom=3cm, left=3cm, right=3cm]{geometry}

\usepackage{booktabs}

\usepackage{color}
\usepackage{latexsym}              
\usepackage{amsmath}               
\usepackage{amssymb}               
\usepackage{amsfonts}              
\usepackage{amsthm}                
\usepackage{multirow}
\usepackage{tikz}                  
\usetikzlibrary{arrows,positioning,shapes}
\usepackage{tcolorbox}

\usepackage[mathcal]{eucal}
\usepackage{cleveref}
\crefname{assumption}{Assumption}{Assumptions}
\crefname{equation}{Eq.}{Eqs.}
\crefname{figure}{Fig.}{Figs.}
\crefname{table}{Table}{Tables}
\crefname{section}{Sec.}{Secs.}
\crefname{theorem}{Thm.}{Thms.}
\crefname{lemma}{Lemma}{Lemmas}
\crefname{corollary}{Cor.}{Cors.}
\crefname{example}{Example}{Examples}
\crefname{appendix}{Appendix}{Appendixes}
\crefname{remark}{Remark}{Remark}

\renewenvironment{proof}[1][\proofname]{{\bfseries #1.}}{\qed \\ }

\makeatother

\newcounter{remark}[section]

\newcommand{\relinterior}{\operatorname{relint}}
\newcommand{\interior}{\operatorname{int}}

\newcommand{\calV}{\mathcal{V}}
\newcommand{\calX}{\mathcal{X}}
\newcommand{\calY}{\mathcal{Y}}
\newcommand{\calZ}{\mathcal{Z}}
\newcommand{\calE}{\mathcal{E}}
\newcommand{\calD}{\mathcal{D}}

\newcommand{\calH}{\mathcal{H}}

\newcommand{\calG}{\mathcal{G}}
\newcommand{\calM}{\mathcal{M}}

\newcommand{\calR}{\mathcal{R}}

\newcommand{\ERL}{\calE}
\newcommand{\ERS}{\calR}

\DeclareMathAlphabet{\mathsfsl}{OT1}{cmss}{m}{sl}

\renewcommand{\phi}{\varphi}

\newcommand{\Rspace}[1]{\mathbb{R}^{#1}}

\newcommand{\R}{\mathbb{R}}
\newcommand{\N}{\mathbb{N}}

\newcommand{\fstar}{f^\star}
\newcommand{\gstar}{g^\star}
\newcommand{\mustar}{\mu^\star}
\newcommand{\fhat}{\widehat{f}}
\newcommand{\ghat}{\widehat{g}}
\newcommand{\muhat}{\widehat{\mu}}

\newcommand*{\defeq}{\mathrel{\vcenter{\baselineskip0.5ex \lineskiplimit0pt
                     \hbox{\scriptsize.}\hbox{\scriptsize.}}}%
                     =}

\newcommand{\eqal}[1]{
\begin{align}
#1
\end{align}
}

\newcommand{\argmin}{\operatorname*{arg\; min}}
\newcommand{\argmax}{\operatorname*{arg\; max}}

\newcommand{\Expect}{\operatorname{\mathbb{E}}}

\theoremstyle{plain}  
\newtheorem{theorem}{Theorem}[section]
\newtheorem{definition}[theorem]{Definition}

\newtheorem{lemma}[theorem]{Lemma}
\newtheorem{proposition}[theorem]{Proposition}

\newtheorem{remark}[theorem]{Remark}

\newtheorem{example}[theorem]{Example}

\begin{document}
\title{A General Theory for Structured Prediction with Smooth Convex Surrogates}

\author{\textbf{Alex Nowak-Vila, Francis Bach, Alessandro Rudi} \\ [2ex]
INRIA - D\'epartement d'Informatique de l'\'Ecole Normale Sup\'erieure \\
PSL Research University \\
 Paris, France \\\\
}

\maketitle

\begin{abstract}
    In this work we provide a theoretical framework for structured prediction that generalizes the existing theory of surrogate methods for binary and multiclass classification based on estimating conditional probabilities with smooth convex surrogates (e.g. logistic regression). The theory relies on a natural characterization of structural properties of the task loss and allows to
    derive statistical guarantees for many widely used methods in the context of multilabeling, ranking, ordinal regression and graph matching. In particular, we characterize the smooth convex surrogates compatible with a given task loss in terms of a suitable Bregman divergence composed with a link function. This allows to derive tight bounds for the calibration function and to obtain novel results on existing surrogate frameworks for structured prediction such as conditional random fields and quadratic surrogates.
\end{abstract}

\section{Introduction}

In statistical machine learning, we are usually interested in predicting an unobserved output element~$y$ from a discrete output space $\calY$ given an observed value $x$ from an input space $\calX$. This is done by estimating a function $f$ such that $f(x)\approx y$ from a finite set of example pairs $(x,y)$.  

In many practical domains such as natural language processing \cite{smith2011linguistic}, computer vision \cite{nowozin2011structured} and computational biology \cite{durbin1998biological}, the outputs are structured objects, such as sequences, images, graphs, etc. 
This structure is implicitly characterized by the loss function $L:\calY\times\calY\rightarrow\Rspace{}$ used to measure the error between the prediction and the observed output as $L(f(x),y)$. Unfortunately, as the outputs are discrete, the direct minimization of the loss function is known to be intractable even for the simplest losses such as the binary 0-1 loss \cite{arora1997hardness}. A common approach to the problem is to design a surrogate loss $S:\calV\times\calY\rightarrow\Rspace{}$ defined in a continuous surrogate space $\calV$ that can be minimized in practice and construct the functional $f$ by ``decoding'' the values from the continuous space to the discrete space of outputs.

In this paper, we construct a general theory for structured output prediction using smooth convex surrogates based on estimating the Bayes risk of the task loss. The methods we consider can be seen as a generalization of binary and multiclass methods based on estimating the conditional probabilities \cite{bartlett2006convexity, zhang2004statistical, zhang2004statisticalbehavior} to general discrete losses, and correspond to proper composite losses \cite{reid2010composite,vernet2011composite} for multiclass classification. Our construction is based on two main ingredients; first, the characterization of the structural properties of a loss function $L$ by means of an affine decomposition of the loss \cite{ramaswamy2013convex, nowak2018sharp}, which we present in \cref{sec:setting}, and second, the Bregman divergence characterization of proper scoring rules for eliciting linear properties of a distribution \cite{abernethy2012characterization, pmlr-v40-Frongillo15}, which has already been noted to have strong links with the design of consistent surrogate losses \cite{agarwal2015consistent}.

We put these two ideas together in \cref{sec:surrogateframework} to construct calibrated surrogates, which are \emph{consistent} smooth convex surrogates with two basic elements, namely, a differentiable and strictly convex \emph{potential} $h$ and a continuous invertible \emph{link function} $t$, which can be easily obtained from the surrogate loss. 
We showcase the generality of our construction by showing how general methods for structured prediction such as the quadratic surrogate \cite{ciliberto2016consistent, osokin2017structured, ciliberto2018localized} and conditional random fields (CRFs)  \cite{Lafferty:2001:CRF:645530.655813, settles2004biomedical}, and widely used methods in multiclass classification \cite{zhang2004statistical}, multilabel classification \cite{read2011classifier}, ordinal regression \cite{pedregosa2017consistency}, amongst others,  fall into our framework. Hinge-type surrogates such as the structured SVM \cite{crammer2001algorithmic}, which is known to be inconsistent \cite{tewari2007consistency}, are not included.

This theoretical framework allows to derive guarantees by relating the surrogate risk associated to $S$ (object that we can minimize) to the actual risk associated to $L$ (object that we want to minimize) by means of convex lower bounds on the \emph{calibration function} $\zeta_h$ \cite{bartlett2006convexity, steinwart2007compare, osokin2017structured}, which is a mathematical object that only depends on the surrogate loss through the potential $h$. In \cref{sec:theoreticalanalysis}, we provide an exact formula for the calibration function (\cref{th:exactcalibration}) and a user-friendly quadratic lower bound for strongly convex potentials (\cref{th:lowerboundcalibration}). 

There, we also analyze the role of the link function on the complexity of the surrogate method by studying the learning guarantees when the convex surrogate is minimized with a stochastic learning algorithm (\cref{th:ASGD}). In particular, we show that, while the relation between excess risks is related to the potential $h$, the approximation error is crucially related to the link function. More specifically, we discuss the benefits of logistic-type surrogates with respect to the quadratic-type ones.

Finally, those results are then used in \cref{sec:examples} to derive learning guarantees for specific methods on multiple tasks for the first time, while also recovering existing results. The most significant novel results on this direction being an exact expression for the calibration function for the quadratic surrogate (\cref{th:exactcalibrationquadraticsimple}) and a quadratic lower bound for CRFs (\cref{prop:calibrationCRFs}). 

\section{Setting}
    \label{sec:setting}
    \subsection{Supervised Learning} \label{sec:supervisedlearning}
    
    The problem of {\em supervised learning} consists in learning from examples a function relating inputs with observations/labels. More specifically, let $\calY$ be the space of observations, denoted {\em observation space} or {\em label space} and $\calX$ be the {\em input space}. The quality of the predicted output is measured by a given {\em loss function} $L$. 
    In many scenarios the output of the function lies in a different space than the observations, for instance in subset ranking losses \cite{chen2009ranking} or losses with an abstain option \cite{ramaswamy2015consistent}. We denote then by $\calZ$ the {\em output space}, so
    \begin{equation}\label{eq:loss_function}
        L:\calZ\times \calY\longrightarrow \R,
    \end{equation}
    where $L(z, y)$ measures the cost of predicting $z$ when the
    observed value is $y$. We assume that $\calY$ and $\calZ$ are discrete.
    Finally the data are assumed to be distributed according to a probability measure $\rho$ on $\calX \times \calY$. The goal of supervised learning is then to recover the function $\fstar$ 
    \footnote{In general $\fstar$ is not unique, as there might be $x\in\calX$ with more than one optimal outputs. For simplicity, we assume that we have a method to choose a unique output between the optimal ones. Note that this is always possible as $\calZ$ is discrete, so one can always construct this method using an ordering of the elements of $\calZ$.} minimizing the {\em expected risk} $\ERL(f)$ of the loss,
    \eqal{\label{eq:optimal_f}
    \fstar = \argmin_{f: \calX \to \calZ}\ERL(f), \quad \ERL(f) =\Expect_{(X,Y)\sim\rho}L(f(X), Y),
    \vspace*{-4pt}
    }
    given only a number of examples $(x_i, y_i)_{i=1}^n$, with $n \in \mathbb{N}$, sampled independently from $\rho$. 
    The quality of an estimator $\fhat$ for $\fstar$ is measured in terms of the {\em excess risk} $\ERL(\fhat) - \ERL(\fstar)$.
    
    It is known that $\fstar$ is characterized as \cite{steinwart2008support,ciliberto2016consistent},
    \begin{equation}\label{eq:bayes_optimum}
        \fstar(x)=\textstyle \argmin_{z\in\calZ}~\ell(z,\rho(\cdot|x)),
    \end{equation}
    where for any $q\in\operatorname{Prob}(\calY)$, the quantity $\ell(z, q) = \Expect_{Y\sim q}L(z,Y)$ is the 
    {\em Bayes risk}, defined as the expectation of the loss with respect to the distribution $q$ on the labels.
    We also define the \emph{excess Bayes risk} as 
    $\delta\ell(z,q)=\ell(z, q) - \min_{z'\in\calZ}\ell(z',q)\geq 0$.
        
    \subsection{Affine Decomposition of Discrete Losses and Marginal Polytope}\label{sec:affinedecomposition}
    
    Consider the following \emph{affine decomposition} of a loss $L$ \cite{ramaswamy2013convex, nowak2018sharp},
    \begin{equation}\label{eq:affinedecomposition}
        L(z,y) = \langle\psi(z), \phi(y)\rangle + c,
    \end{equation}
    where $\psi:\calZ \to \calH$ and $\phi: \calY \to \calH$ are embeddings to a vector space $\calH$ with Euclidean scalar product $\langle\cdot,\cdot\rangle$ and $c\in\Rspace{}$ is a scalar constant. 
    Note that by linearity of the inner product,
    \begin{equation*}
        \ell(z,q)
    = \Expect_{Y\sim q}\langle\psi(z),\phi(Y)\rangle + c
    = \langle\psi(z),\mu(q) \rangle + c,
    \end{equation*}
    with $\mu(q) = \Expect_{Y\sim q}\phi(Y)$ the vector of moments of the statistic $\phi$. If we denote $\mustar(x)=\mu(\rho(\cdot|x))$, then the excess Bayes risk takes the form 
    $\delta\ell(z,\rho(\cdot|x))=\langle\psi(z)-\psi(\fstar(x)), \mustar(x)\rangle\geq 0$.
   Note that the affine decomposition always exists, is not unique and it corresponds to a low-rank decomposition of the ``centered'' loss matrix $L-c\in\Rspace{\calZ\times\calY}$. The image of $\mustar$ lies inside the convex hull of the $\phi(y)'s$, that is,
    \begin{equation}
        \text{Im}(\mustar) \subseteq  \calM \defeq \operatorname{hull}(\phi(\calY)) \subset \calH.
    \end{equation}
    The set $\calM$ is the polytope corresponding to the convex hull of the finite set
    $\phi(\calY)\subset\calH$. We will refer to $\calM$ as the \emph{marginal polytope} associated to the statistic $\phi$, making an analogy to the literature on graphical models 
    \cite{wainwright2008graphical}. We denote by $k=\operatorname{dim}(\calH)$ the dimension of the embedding space and by $r=\operatorname{dim}(\calM)$ the dimension of the marginal polytope, defined as the dimension of its affine hull. Note that it can be the case that $r<k$, which means that $\calM$ is not full-dimensional in $\calH$. 
    
    \begin{example}[Multiclass and multilabel classification]\label[example]{ex:multiclassmultilabel}
    The 0-1 loss used for $k$-multiclass classification ($\calZ=\calY=\{1,\ldots,k\}$) can be decomposed as
    $L(z,y)=1(z\neq y) = 1 - \langle e_z,e_y\rangle$, where $\calH=\Rspace{k}$ and $e_z$ is the $z$-th vector of the canonical basis in $\Rspace{k}$. In this case, the loss matrix is full-rank and the marginal polytope is the simplex in $k$ dimensions, $\calM=\operatorname{hull}(\{e_y\}_{y=1}^k)=\Delta_k$, which is not full-dimensional and has dimension $r=k-1$.
    Another example is the Hamming loss used for multilabel classification ($\calZ=\calY=\{-1,1\}^k$). In this case, the loss matrix is extremely low-rank and can be decomposed as
    $L(z,y) = \frac{1}{k}\sum_{j=1}^k1(z_j\neq y_j)=\frac{1}{2}-\langle z/(2k), y\rangle$, where $k=\log|\calY|$.
    The marginal polytope is the cube $\calM=\operatorname{hull}(\{-1,1\}^k)=[-1,1]^k$ which is full dimensional in $\calH=\Rspace{k}$.
    \end{example}
    

\section{Surrogate Framework} \label{sec:surrogateframework}
    \subsection{Estimation of the Bayes Risk with Surrogate Losses}\label{sec:bayesrisk}
    
    The construction in \cref{sec:affinedecomposition} leads to a natural method in order to estimate $\fstar$ based on estimating the conditional expectation $\mustar$. Indeed, given an estimator $\muhat$ of $\mustar$, one can first construct an estimator of the Bayes risk as $\widehat{\ell}(z, \rho(\cdot|x)) \defeq \langle\psi(z), \muhat(x)\rangle + c$, and then define the resulting estimator as
    \begin{equation}
        \fhat(x) \defeq \argmin_{z\in\calZ}~\widehat{\ell}(z,\rho(\cdot|x)) = \argmin_{z\in\calZ}~ \langle\psi(z), \muhat(x)\rangle.
    \end{equation}
    In the following, we study a framework to construct estimators of $\mustar$ using surrogate losses. We consider estimators which are based on the minimization of the \emph{expected surrogate risk} $\ERS(g)$ of a \emph{surrogate loss} $S:\calV\times\calY\rightarrow\Rspace{}$ defined in a (unconstrained) vector space $\calV$,
        \begin{equation}\label{eq:surrogaterisk}
            \gstar = \underset{g:\calX\rightarrow\calV}{\argmin}~
            \ERS(g), \quad \ERS(g) = \Expect_{(X,Y)\sim\rho}S(g(X),Y).
        \end{equation}
        
        An estimator $\fhat$ of $\fstar$ is built from an estimator $\ghat$ of $\gstar$ using a
        \emph{decoding mapping} $d:\calV\rightarrow\calZ$ as~$\fhat=d\circ\ghat$.
        The pair $(S,d)$ constitutes a \emph{surrogate method} and
        we say that it is \emph{Fisher consistent}~\cite{lin2004note} to the loss $L$ if the minimizer of the expected surrogate risk \eqref{eq:surrogaterisk} leads to the minimizer of the true risk \eqref{eq:optimal_f} as $\fstar=d\circ\gstar$.
        
        Analogously to the quantities defined in \cref{sec:supervisedlearning} for the discrete loss $L$, we define the \emph{excess surrogate risk} as $\ERS(\ghat)-\ERS(\gstar)$, the \emph{Bayes surrogate risk}
        $s(v,q)=\Expect_{Y\sim q}S(v,Y)$ and the  \emph{excess Bayes surrogate risk} as $\delta s(v,q)=s(v,q)-\min_{v'\in\calV}s(v',q)\geq 0$. Similarly to \cref{eq:bayes_optimum}, $\gstar$ is characterized by $\gstar(x) = \argmin_{v\in\calV}s(v,\rho(\cdot|x))$, which we assume unique.
        
    
        We will now focus on surrogate losses for which $\mustar$ can be computed from the minimizer $\gstar$ through a continuous \emph{injective} mapping $t:\calM\rightarrow\calV$ called the \emph{link function}. More precisely, we ask
        \begin{equation}\label{eq:optimal_g}
        t(\mu(q))=\argmin_{v\in\calV}~s(v,q), \quad \forall q\in\operatorname{Prob}(\calY).
        \end{equation}
        Although \cref{eq:optimal_g} is the only property that we need from $S$ in order to build the theoretical framework, we assume in the following that $S$ is \emph{smooth} and \emph{convex}. This is justified in \cref{rk:smoothnessconvexity}.
        \begin{remark}[On smoothness and convexity requirement on $S$]\label[remark]{rk:smoothnessconvexity}
        Although we do not formalize any statement of that kind, the smoothness of $S$ is closely related to the injectivity of $t$. For instance, in the multiclass case where $\mu(q)=q$, if $v_0=\argmin_{v\in\calV}s(v,q_0)$ and $s(\cdot,q_0)$ is not differentiable at $v_0$, then one can find $q'\neq q_0$ such that $s(v_0,q')=s(v_0,q_0)$, and so the link is not injective. This is the case for hinge-type surrogates, which do not estimate conditional probabilities. A proper analysis in this direction can be formalized in terms of supporting hyperplanes on the so-called superdiction set associated to $S$ (see Sec. 5.3 in \cite{vernet2011composite}). The convexity requirement is made in order to be able to minimize in a tractable way the expected surrogate risk.
        \end{remark}
        If a surrogate loss satisfies \cref{eq:optimal_g}, then one can relate $\fstar$ and $\gstar$ using the decoding mapping $d_{\psi,t}:t(\calM)\rightarrow\calZ$ defined as
        \begin{equation}\label{eq:decoding}
            d_{\psi,t}(v) =\textstyle \argmin_{z\in\calZ}~\langle\psi(z), t^{-1}(v)\rangle.
        \end{equation}
        The role of the link function here is to deal with the fact that the image of $\mustar$ lives in $\calM$, which is a constrained, bounded, and possibly non full-dimensional set of $\calH$.  As in general it is not easy to impose a structural constraint on the hypothesis space,
        the goal of the link function is to encode this geometry by mapping points from a ``simpler" $t(\calM)\subseteq\calV$ to $\calM$. Note that $d_{\psi,t}$ is defined in $t(\calM)$, so if $t(\calM)\neq\calV$, we do not know how to map points from $\calV\backslash t(\calM)$ to $\calH$.
        In the next \cref{sec:BDrepresentation} we show that in the cases where $t(\calM)\neq\calV$, the link can be sometimes naturally extended to cover the whole vector space $\calV$. In order to do this, we first show that surrogates satisfying \cref{eq:optimal_g} have a very rigid structure in $t(\calM)$ in the form of a Bregman divergence representation. Then, we define $\phi$-calibrated surrogates as the ones such that the corresponding Bregman divergence representation can be extended to $\calV$.
        
        Assume for now that $t(\calM)=\calV$. The surrogate method $(S,d_{\psi,t})$ works as follows; in the learning phase, an estimator $\ghat$ is found by (regularized) empirical risk minimization on the smooth convex surrogate loss $S$, and then, given a new input element $x$, the decoding mapping $d_{\psi,t}$ computes the prediction $\fhat(x)$ from $\ghat$. Note that the computational complexity of inference can vary depending on the loss $L$ (see \cite{nowak2018sharp,ciliberto2016consistent}). See boxes below.
        \vspace{5pt}
        \begin{tcolorbox}[colframe=black!25, colback=black!2] \vspace{-5pt}
        \begin{center}
            \textbf{Learning}
        \end{center}
        \vspace{-9pt}
        \hrule
        \noindent
        \begin{itemize}
            \item[-] \textit{Given}: a functional hypothesis space $\calG\subset\{g:\calX\rightarrow\calV\}$, dataset $(x_i,y_i)_{1\leq i\leq n}$ and surrogate loss $S:\calV\times\calY\rightarrow\Rspace{}$.
            \vspace{-5pt}
            \item[-] \textit{Goal}: Minimize the expected surrogate risk $\ERS(g)$ as:
        \end{itemize}
        \vspace{-5pt}
            \begin{equation}\label{eq:learning}
                \widehat{g} = \argmin_{g \in \calG} \frac{1}{n}\sum_{i=1}^n S(g(x_i), y_i) + \lambda \|g\|^2_\calG.
            \end{equation}
        \vspace{-10pt}
        \end{tcolorbox}
        \begin{tcolorbox}[colframe=black!25, colback=black!2] \vspace{-5pt}
        \begin{center}
            \textbf{Inference}
        \end{center}
        \vspace{-7pt}
        \hrule
        \noindent
        \begin{itemize}
            \item[-] \textit{Given}: an input element $x\in\calX$, an estimator $\ghat\in\calG$, the inverse of the link function $t^{-1}:\calV\rightarrow\calH$ and an embedding $\psi:\calZ\rightarrow\calH$.
            \vspace{-5pt}
            \item[-] \textit{Goal}: Construct prediction $\fhat(x)\in\calZ$ as:
        \end{itemize}
        \vspace{-5pt}
        \begin{equation}\label{eq:inference}
            \widehat{f}(x) = d_{\psi,t}\circ\ghat(x) =  \argmin_{z \in\calZ}~\left\langle \psi(z),t^{-1}(\widehat{g}(x)) \right\rangle.
        \end{equation}
        \vspace{-10pt}
        \end{tcolorbox}
        
        \subsection{Bregman Divergence Representation}\label{sec:BDrepresentation}
        Let $\calD\subseteq\calH$ be a convex set. Recall that the Bregman divergence (BD) associated to a convex and differentiable function $h:\calD\subseteq\calH\xrightarrow{}\Rspace{}$ is defined as
        \begin{equation}\label{eq:bregmandivergence}
            D_{h}(u', u) = h(u') - h(u) - \langle u' - u, \nabla h(u)\rangle.
        \end{equation}
        We will say that a surrogate loss $S$ has a BD representation if the excess Bayes surrogate risk $\delta s(v,q)$ can be written as a BD by composition with the link function.
    \begin{definition}[BD Representation] \label[definition]{def:BDrepresentation} The surrogate loss $S$ has a $(h,t,\phi)$-BD representation in $\calV'\subset\calV$, if there exists a set $\calD\supseteq\calM$ containing the marginal polytope, a strictly convex and differentiable potential $h:\calD\subseteq\calH\rightarrow\Rspace{}$ and continuous invertible link $t:\calD\rightarrow\calV'$, such that the excess Bayes surrogate risk can be written as
    \begin{equation}
        \delta s(v,q) = D_h(\mu(q), t^{-1}(v)), \quad \forall v\in \calV'\subset\calV, \forall q\in\operatorname{Prob}(\calY).
    \end{equation}
    \end{definition}
    
    The following \cref{th:compositerepresentation} states that any surrogate loss satisfying \cref{eq:optimal_g} has a BD representation in $t(\calM)$, which justifies why we focus on these representations of losses.
    
    \begin{theorem}[BD Representation in $t(\calM)$]\label[theorem]{th:compositerepresentation}
        If the surrogate loss $S:\calV\times\calY\rightarrow\Rspace{}$ is continuous and satisfies 
        \cref{eq:optimal_g} for a continuous injective mapping $t:\calM\rightarrow\calV$, then
        it has a $(h,t,\phi)$-BD representation in $t(\calM)\subseteq\calV$.
    \end{theorem}
    The proof of \cref{th:compositerepresentation} can be found in \cref{app:BDrepresentation} and it is based on a characterization of scoring rules for linear properties of a distribution as Bregman divergences associated to strictly convex functions \cite{abernethy2012characterization, frongillo2015vector}. The differentiability of $h$ is derived from the continuity of the link $t$ and $S$.
    
    It is important to highlight the fact that the function $h$ is defined up to an additive affine term, as the BD is invariant under this transformation. Hence, we will say that $h$ and $h'$ are equivalent if and only if $h-h'$ is an affine function. Note that the function $h$ given by \cref{th:compositerepresentation} can be computed as
    \begin{equation}\label{eq:computingh}
        h(\mu(q)) = \delta s(v_0, q),
    \end{equation}
    for any $v_0\in t(\calM)$.
    Indeed, by \cref{th:compositerepresentation}, the dependence on $q$ of $\delta s(v_0, q)$ is only through the vector of moments $\mu(q)$ and $\delta s(v, q)-\delta s(v', q)$ is an affine function of $\mu(q)$, $\forall v,v'\in t(\calM)$.

    Observe that different surrogate losses can yield the same BD representation in $t(\calM)$. For instance, in binary classification, the square, squared hinge and modified Huber margin losses have the same BD representation in $[-1,1]$ \cite{zhang2004statisticalbehavior} (see \cref{app:binaryclassification}).
    
    \subsection{$\phi$-Calibrated Surrogates}
    
    Now, we define the concept of a $\phi$-calibrated loss by asking the surrogate loss satisfying \cref{eq:optimal_g} to extend (in the case that $t(\calM)\neq\calV$) its $(h,t,\phi)$-BD representation in $t(\calM)$ given by \cref{th:compositerepresentation} to $\calV$, which will allow us to define the decoding mapping $d_{\psi,t}$ to the whole vector space $\calV$.
    
    \begin{definition}[$\phi$-Calibrated Surrogates]\label[definition]{def:phicalibrated}
    Let $\phi:\calY\rightarrow\calH$.
    A smooth convex surrogate loss~$S:\calV\times\calY\rightarrow\Rspace{}$ is $\phi$-calibrated if it has a $(h,t,\phi)$-BD representation in the vector space $\calV$.
    \end{definition}
    
    There are many ways of building a continuous extension of $t$ to an invertible mapping in $\calV$ (and thus to extend $d_{\psi,t}$), however, the BD representation extension allows to prove guarantees for estimators $\ghat$ with $\operatorname{Im}(\ghat)\not\subset t(\calM)$ (see \cref{sec:theoreticalanalysis}).
    In general, it is not true that any surrogate loss satisfying \cref{eq:optimal_g} with $t(\calM)\subsetneq\calV$ has an extended BD representation in $\calV$, this is the case for squared hinge and modified Huber margin losses in binary classification (see \cref{app:binaryclassification}).
    \begin{example}[Quadratic, logistic and hinge surrogates]\label{ex:QSandlog}
    Let us provide some examples in binary classification where $\calM=\Delta_2\subset\calH=\Rspace{2}$.
    The quadratic surrogate is defined as 
    $S(v,y) = 1/2\cdot\|v-e_y\|_2^2$ with $\calV=\Rspace{2}$ and satisfies $q = \argmin_{v\in\Rspace{2}}s(v,q)$. It has $h(u)=1/2\cdot\|u\|_2^2$ and the link is $t=Id$. Note that although $t(\Delta_2)=\Delta_2\subsetneq\Rspace{2}$, the BD representation can be extended to $\Rspace{2}$, and in this case $\calD=\Rspace{2}$ and so it is $\phi$-calibrated. The logistic corresponds to
    $S(v,y) = \log(1 + e^{-yv})$ with $\calV=\Rspace{}$ and satisfies $\log(q_1/(1-q_1)) = \argmin_{v\in\Rspace{}}s(v,q)$. In this case the potential is minus the entropy $h(q)=-\operatorname{Ent}(q)$ and the link is $t(q)=\log(q_1/(1-q_1))$ with inverse $t^{-1}(v)=(1+e^{-v})^{-1}$. Note that we have $t(\Delta_2)=\Rspace{}$, so it is $\phi$-calibrated. Finally, consider the hinge margin loss $S(v,y)=\max(1-yv,0)$ with $\calV=\Rspace{}$, which satisfies 
    $\operatorname{sign}(2q_1-1) = \argmin_{v\in\Rspace{}}s(v,q)$, hence, $t$ is not injective, so $S$ is not $\phi$-calibrated.
    \end{example}
    
    Note that if a loss $S$ is $\phi$-calibrated for a statistic $\phi$, then the surrogate method $(S,d_{\psi,t})$ is Fisher consistent w.r.t $L(z,y)=\langle\psi(z),\phi(y)\rangle + c$. This implies that a $\phi$-calibrated loss can be used to consistently minimize different losses by simply changing the embedding $\psi$ at inference time. For instance, if $S$ is $\phi$-calibrated for the statistic $\phi(y)=e_y\in\Rspace{\calY}$, then it can be made consistent for any cost-sensitive matrix loss
    $L\in\Rspace{\calZ\times \calY}$ by setting $\psi(z)=L_z$, where $L_z$ is the $z$-th row of $L$. Indeed, in this case $\calM=\Delta_{\calY}$, and so one can estimate the Bayes risk of any loss with labels $\calY$.
    
    \vspace{5pt}
    \begin{tcolorbox}[colframe=black!25, colback=black!2] 
    \textbf{Summary. } The surrogate loss has two components; the potential $h:\calD\rightarrow\Rspace{}$ and the invertible link function $t:\calD\rightarrow\calV$, which compose the surrogate loss $S$. 
    In the learning phase (see \cref{eq:learning}), only the surrogate loss is needed to minimize $\ERS(g)$, while in the inference phase (see \cref{eq:inference}), one needs the inverse of the link to construct an estimate of $\mustar$, and the rest of the inference only depends on $\psi$.
    The potential function $h$ is not needed to define the surrogate method but it is the mathematical object providing the guarantees in order to relate both excess risks in 
    \cref{sec:theoreticalanalysis}. The link function also has implications in terms of learning complexity (see discussion in \cref{sec:ASGD}). See \cref{fig:diagram} in \cref{app:BDrepresentation} for an illustrative diagram.
    \end{tcolorbox}
    
    We now provide a recipe on how to check whether a surrogate loss is $\phi$-calibrated and to compute its corresponding $(h,t,\phi)$-BD representation if applicable.

    \paragraph{Computing the BD representation and checking $\phi$-calibration.} Given a statistic $\phi:\calY\rightarrow\calH$ and a surrogate $S:\calV\times\calY\rightarrow\Rspace{}$, the first thing to do is to check whether the minimizer of $s(v,q)$ satisfies \cref{eq:optimal_g} for a continuous injective $t$. If this is the case, the potential $h$ can be found up to an additive affine term by \cref{eq:computingh}. If $t(\calM)=\calV$, then $S$ is $\phi$-calibrated. Otherwise, one has to check if there exists an extension of $t$ and $h$ such that
    $\delta s(v, q) = D_h(\mu(q), t^{-1}(v))$ for all $v\in\calV,q\in\operatorname{Prob}(\calY)$. We provide numerous examples in \cref{sec:examples} and in the Appendix.
    
    We present now a special group of $\phi$-calibrated surrogates, whose potential $h$ is a function of Legendre-type and the link is the gradient of the potential $\nabla h$.
    \paragraph{$\phi$-Calibrated surrogates of Legendre-type.} 
    A function $h$ is of Legendre-type in $\calD\subseteq\calH$ if it is strictly convex in $\interior(\calD)$ and essentially smooth, which in particular requires $\lim_{u\to\partial\calD}\|\nabla h(u)\|_2=+\infty$, where $\partial\calD$ is the boundary of $\calD$.
    Given a Legendre-type function $h$ with domain $\calD\supseteq\calM$ including the marginal polytope, one can set the link function to $t=\nabla h$. We call it the \emph{canonical link}. It has the nice property that if $h$ is of Legendre-type in $\calD$, then its Fenchel conjugate $h^*$ is also of Legendre-type and its gradient is the inverse of the link function  $\nabla h^* = (\nabla h)^{-1}$. We denote the resulting loss $S:\operatorname{dom}(h^*)\times\calY\rightarrow\Rspace{}$ a \emph{surrogate loss of Legendre-type}, which is \emph{convex} and has the form:
    \begin{equation}\label{eq:legendretypeloss}
        S(v,y) = D_h(\phi(y), \nabla h^*(v)) = h^*(v) + h(\phi(y)) -
        \langle\phi(y), v\rangle,
    \end{equation}
    The excess Bayes surrogate risk can be written as a BD also in $\operatorname{dom}(h^*)$ as:
    \begin{equation}
        \delta s(v,q) = D_h(\mu(q), \nabla h^*(v)) = D_{h^*}(v, \nabla h(\mu(q))).
    \end{equation}
    Moreover, $\calD$ is bounded if and only
    $\operatorname{dom}(h^*)$ is a vector space and $h^*$ is Lipschitz. Those losses were studied by \cite{blondel2019learning} as a subset of Fenchel-Young losses, but without providing learning guarantees. The most important examples are the quadratic surrogate, where $\calD=\calH$, and CRFs, where $\calD=\calM$, both studied in detail in \cref{sec:QSandCRFs}. Further details on this construction can be found in \cref{app:canonicallink}.
    
\section{Theoretical Analysis}\label{sec:theoreticalanalysis}
    We know by construction that $\phi$-calibrated surrogate losses lead to Fisher consistent surrogate methods $(S,d_{\psi,t})$, which means that the minimizer of the surrogate risk $\ERS$ provides the minimizer of the true risk $\ERL$ as $\fstar=d\circ \gstar$. However, in practice we will never be able to minimize the surrogate risk to optimality. 
    The goal of this section is to \emph{calibrate} the excess surrogate risk to the true excess risk, i.e., quantify how much the excess surrogate risk has to be minimized so that the excess true risk is smaller than $\varepsilon$. This quantification is made by means of the \emph{calibration function} \cite{steinwart2007compare, osokin2017structured, bartlett2006convexity, duchi2010consistency}, which is the mathematical object that will allow us to relate the quantity we can directly minimize to the one that we are ultimately interested in. All the proofs from this section can be found in \cref{app:calibrationrisks}.
    \subsection{Calibrating Risks with the Calibration Function}
     The calibration function is defined as the largest function $\zeta:\Rspace{}_{+}\longrightarrow\Rspace{}_{+}$ that relates both excess Bayes risks as $\zeta(\delta\ell(d(v), q)) \leq \delta s(v, q)$, $\forall v\in\calV, \forall q\in\operatorname{Prob}(\calY)$.
    The calibration function for general losses is thus defined as follows.
    \begin{definition}[Calibration function \cite{osokin2017structured}]\label[definition]{def:calibrationfunction}
        The calibration function $\zeta:\Rspace{}_{+}\longrightarrow\Rspace{}_{+}$ is defined for~$\varepsilon\geq 0$ as the infimum of the excess Bayes surrogate risk when the excess Bayes risk is at least $\varepsilon$:
        \begin{equation}\label{eq:optimalcalibration}
            \zeta(\varepsilon) = \inf\delta s(v, q)\quad \text{such that} \quad
            \delta\ell(d(v), q)
            \geq\varepsilon,~q\in\operatorname{Prob}(\calY),~v\in\calV.
    \end{equation}
    We set $\zeta(\varepsilon)=\infty$ when the feasible set is empty.
    \end{definition}
    Note that $\zeta$ is non-decreasing on $[0, +\infty)$, not necessarily convex 
    (see Example 5 by \cite{bartlett2006convexity})
    and also $\zeta(0)=0$. Note that a larger $\zeta$ is better because we want a large $\delta s(v,q)$ to incur small $\delta\ell(d(v),q)$.
    
The calibration function $\zeta$ relates conditional risks. In order to calibrate risks $\ERS$ and $\ERL$ one needs to impose convexity so that the expectation with respect to the marginal distribution $\rho_{\calX}\in\operatorname{Prob}(\calX)$ can be moved outside of the calibration function. In \cref{th:calibrationrisks}, which can be found in \cite{osokin2017structured}, we calibrate the risks by taking a convex lower bound of $\zeta$.
\begin{theorem}[Calibration between risks \cite{osokin2017structured}]\label[theorem]{th:calibrationrisks} Let $\bar{\zeta}$ be a convex lower bound of $\zeta$.
We have 
\begin{equation}\label{eq:riskcalibration}
\bar{\zeta}(\ERL(d\circ\ghat) - \ERL(\fstar))\leq \ERS(\ghat) - \ERS(\gstar)
\end{equation}
for all $\ghat:\calX\rightarrow\calV$.
The tightest convex lower bound $\bar{\zeta}$ of $\zeta$ is its lower convex envelope which is defined by the Fenchel bi-conjugate $\zeta^{**}$ \footnote{The Fenchel bi-conjugate is characterized by $\operatorname{epi}(\zeta^{**})=\overline{\operatorname{hull}(\operatorname{epi}(\zeta}))$, where $\operatorname{epi}(\zeta)$ denotes the epigraph of the function $\zeta$ and 
$\overline{\operatorname{hull}(A)}$ is the closure of the convex hull of the set $A$.}.
\end{theorem}

Note that a surrogate method is Fisher consistent if and only if $\zeta^{**}(\varepsilon)>0$ for all $\varepsilon>0$, as this implies $\fstar=d\circ\gstar$. In the case that $\zeta^{**}\neq\zeta$, this property also translates to $\zeta$. See \cref{fig:calibrationfunctionandhamming}.
\subsection{Calibration Function for $\phi$-Calibrated Losses}
The computation of $\zeta$ (or a convex lower bound thereof) is known not to be easy and has been a central topic of study for many past works 
\cite{bartlett2006convexity, pires2013cost, osokin2017structured}.
 One of the main contributions of this work is to provide an exact formula for $\zeta$ for $\phi$-calibrated losses based on Bregman divergences between pairs of sets in $\calH$. This geometric interpretation of the calibration function will be used to compute the calibration function for existing surrogates which are widely used in practice.

First, let us define the \emph{calibration sets} $\calH_\varepsilon(z)$ for every $\varepsilon\geq 0$ and $z\in\calZ$ as
\begin{equation}
    \calH_\varepsilon(z) = \{u\in\calH~|~\langle\psi(z)-\psi(z'), u\rangle\leq \varepsilon, \forall z'\in\calZ\} \subset\calH.
\end{equation}
The points in $\calH_\varepsilon(z)$ are the ones whose Bayes risk is at least $\varepsilon$-close to have $z$ as optimal prediction. In particular, $\calH_0(z)$ is the set of points with optimal prediction $z$, which can be equivalently written as $\mustar(x)\in\calH_0(\fstar(x))$, $\forall x\in\calX$.
Note that $\calH_\varepsilon(z)$ is convex $\forall \varepsilon\geq 0, \forall z\in\calZ$. See \cref{fig:calibrationfunctionandhamming} for a visualization of the calibration sets for the Hamming loss in the context of multilabel classification (and \cref{fig:multiclass_calibration} in \cref{app:multiclassclassification} for the 0-1 loss for multiclass classification). 

\begin{theorem}[Calibration function for $\phi$-calibrated losses]\label[theorem]{th:exactcalibration}
Let $S$ be a $\phi$-calibrated surrogate with potential function $h:\calD\rightarrow\Rspace{}$ and let $L(z,y)=\langle\psi(z),\phi(y)\rangle+c$. The calibration function only depends on $S$ through $h$ and we denote it by $\zeta_h$. Moreover, it can be written as
\begin{equation}\label{eq:exactcalibration}
\zeta_h(\varepsilon) = \min_{z\in\calZ}D_h(\calH_{\varepsilon}(z)^c\cap\calM, \calH_{0}(z)\cap\calD),
\end{equation}
where the Bregman divergence between sets $A,B$ is defined as
$D_h(A,B) = \inf_{u\in A,v\in B}D_h(u, v)$.
\end{theorem}

Note that the Bregman divergence inside the minimum in \cref{eq:exactcalibration} does not lead to a convex minimization problem since $\calH_{\varepsilon}(z)^c\cap\calM$ is not convex and $D_h(u,v)$ is in general not jointly convex in $(u,v)$, with notable exceptions such as the KL-divergence and squared distance \cite{bauschke2001joint}. In general, the exact computation of $\zeta_h$ using \cref{th:exactcalibration} can still be hard to perform, for instance, when the embeddings $\psi$ are not simple to work with or the problem lacks symmetries. In the 
following \cref{th:lowerboundcalibration} we provide a user-friendly lower bound when the potential $h$ is strongly convex.
Recall that a function $h$ is $({1}/{\beta_{\|\cdot\|}})$-strongly convex w.r.t a norm $\|\cdot\|$ in $\calD$ if it satisfies
$h(u) \geq h(v) + \langle u-v, \nabla h(v)\rangle + \frac{1}{2\beta_{\|\cdot\|}}\|u-v\|^2$, $\forall u,v\in\calD$.

\begin{theorem}[User-friendly lower bound on $\zeta_h$]\label[theorem]{th:lowerboundcalibration}
Let $\zeta_{h}(\varepsilon)$ be the calibration function given by \cref{eq:exactcalibration}. If $h$ is $({1}/{\beta_{\|\cdot\|}})$-strongly convex w.r.t a norm $\|\cdot\|$ in $\calD$, then:
\begin{equation}
    \zeta_h(\varepsilon) \geq \ \frac{\varepsilon^2} {8c_{\psi,\|\cdot\|_{*}}^2\beta_{\|\cdot\|}}, 
\end{equation}
where $c_{\psi,\|\cdot\|_{*}} = \sup_{z\in\calZ}\|\psi(z)\|_{*}$ and $\|\cdot\|_{*}$ is the dual norm of $\|\cdot\|$.
\end{theorem}
The proof is provided in \cref{app:lowerboundcalibration}, together with \cref{th:improvedlowerbound}, that gives a tighter bound in the case of strong convexity w.r.t the Euclidean norm.
Finally, the following \cref{th:upperboundcalibration}, states that $\zeta_h$ can never be larger than a quadratic for $\phi$-calibrated surrogates.
\begin{theorem}[Existence of quadratic upper bound]\label[theorem]{th:upperboundcalibration}
Assume $h$ is twice differentiable. Then, the calibration function $\zeta_h$ is upper bounded by a quadratic close to the origin, i.e., $\zeta_h(\varepsilon) = O(\varepsilon^2)$.
\end{theorem}

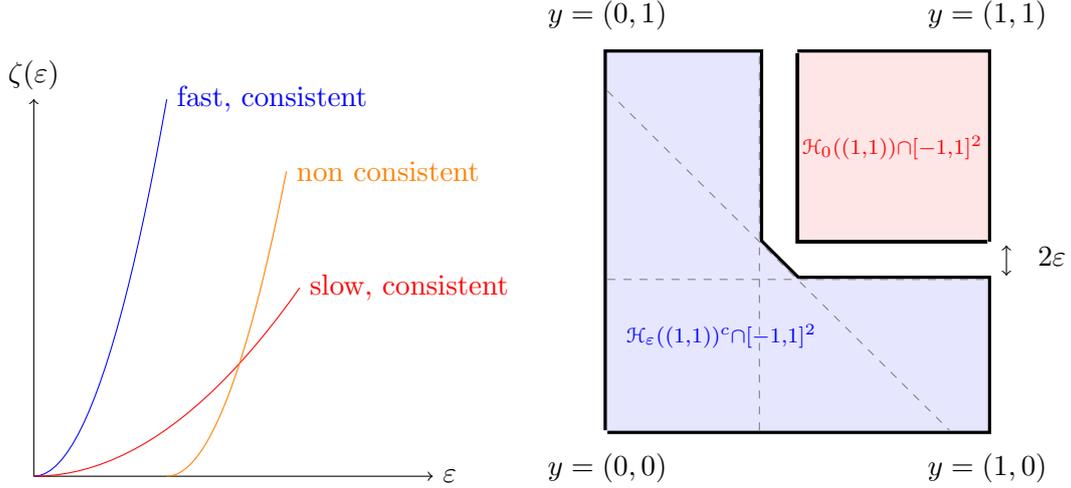
\begin{figure}[ht!]
    \centering
    \begin{tikzpicture}[domain=0:1, xscale=3.5, yscale=5]
    \draw[->] (0,0) -- (1.5,0) node[right] {$\varepsilon$};
    \draw[->] (0,0) -- (0,1) node[above] {$\zeta(\varepsilon)$};
    \draw[color=blue, domain=0:0.5] plot (\x, {4*\x^2}) 
        node[right] {fast, consistent};
    \draw[color=red, domain=0:1] plot (\x, {0.5*\x^2}) 
        node[right] {slow, consistent};
    \draw[color=orange, domain=0.5:0.95] plot (\x, {4*(\x-0.5)^2}) 
        node[right] {non consistent};
\end{tikzpicture}
    \begin{tikzpicture}[domain=0:1, xscale=5, yscale=5]
    \draw[line width=1mm, color=black] (1,0.5) -- (1, 1) -- (0.5, 1);
    \draw[fill, color=red!10] (1, 0.5) -- (0.5, 0.5) -- (0.5, 1)
    -- (1,1) -- (1, 0.5);
    \draw[line width=0.5mm, color=black] (1, 0.5) -- (0.5, 0.5) -- (0.5, 1);
    
    \draw[line width=1mm, color=black] (0,0) -- (1, 0) -- (1, 0.4) -- (0.5, 0.4) -- (0.4,0.5) -- (0.4, 1)
    -- (0,1) -- (0, 0);
    \draw[fill, color=blue!10] (0,0) -- (1, 0) -- (1, 0.4) -- (0.5, 0.4) -- (0.4, 0.5) -- (0.4, 1)
    -- (0,1) -- (0, 0);
    \node[color=red] at (0.75, 0.75) {${\scriptstyle \calH_0((1,1))\cap[-1,1]^2}$};
    \node[color=blue] at (0.3, 0.25) {${\scriptstyle \calH_{\varepsilon}((1,1))^c\cap[-1,1]^2}$};
    \draw[dashed, color=black!50] (0.4,0) -- (0.4,1); 
    \draw[dashed, color=black!50] (0, 0.4) -- (1,0.4); 
    \draw[dashed, color=black!50] (0, 0.9) -- (0.9, 0); 
    \draw[<->] (1.05, 0.41) -- +  (90:0.08);
    \node at (1.17, 0.46) {$2\varepsilon$};
    \node at (0, -0.1) {$y=(0,0)$};
    \node at (1, -0.1) {$y=(1,0)$};
    \node at (1, 1.1) {$y=(1,1)$};
    \node at (0, 1.1) {$y=(0,1)$};
\end{tikzpicture}
    \caption{\textbf{Left:} Both red and blue curves correspond to calibration functions of Fisher consistent surrogate methods as $\zeta(\varepsilon)>0$ for all $\varepsilon>0$, which is not the case for the orange curve. However, the blue curve has better guarantees because the same surrogate excess risk leads to smaller excess true risk. \textbf{Right:} Illustration of the sets $\calH_0(z)\cap\calM$ and $\calH_{\varepsilon}(z)^c\cap\calM$ for the Hamming loss with $k=2$ labels and $z=(1,1)$. In this case, by symmetry, the calibration function is computed as the Bregman divergence between these two sets.}
    \label{fig:calibrationfunctionandhamming}
\end{figure}

\subsection{Improved Calibration under Low Noise Assumption}
The result of \cref{th:calibrationrisks} can be further improved under low noise assumptions on the marginal distribution $\rho_{\calX}$. Following the definition from classification \cite{bartlett2006convexity, mroueh2012multiclass, zhang2004statistical}, we define the \emph{margin function} $\gamma:\calX\rightarrow\Rspace{}$ as 
$\gamma(x) = \min_{z'\neq \fstar(x)} \delta\ell(z', \rho(\cdot|x))$.
We say that the \emph{$p$-noise condition} is satisfied if
\begin{equation}\label{eq:tsybakov}
\rho_{\calX}(\gamma(X)\leq\varepsilon)=o(\varepsilon^p).
\end{equation} 
A simple computation shows that \cref{eq:tsybakov} holds if and only if $\|1/\gamma\|_{L_p(\rho_{\calX})}=\gamma_p<\infty$ \cite{steinwart2011estimating}.
\begin{theorem}[Calibration of Risks under low noise and hard margin assumption]
\label[theorem]{th:calibrationriskslownoise}
Let $\bar{\zeta}$ be a convex lower bound of $\zeta$.
\begin{itemize}
    \item[(1)] If the $p$-noise condition \eqref{eq:tsybakov} is satisfied, we have that $\bar{\zeta}^{(p)}$ defined as
    \vspace{-4pt}
        \begin{equation}\label{eq:improved_calibration}
            \bar{\zeta}^{(p)}(\varepsilon) = 
            (\gamma_p\varepsilon^p)^{\frac{1}{p+1}} \bar{\zeta}
            ((\gamma_p^{-1}\varepsilon)^{\frac{1}{p+1}}/2),
        \vspace{-4pt}\end{equation}
        satisfies \cref{eq:riskcalibration} where $\|1/\gamma\|_{L_p(\rho_{\calX})}=\gamma_p$. Moreover, we have that
        $\bar{\zeta}^{(p)}\gtrsim\bar{\zeta}$. Hence, $\bar{\zeta}^{(p)}$ never provides a worse rate than $\bar{\zeta}$.
    \item[(2)] If $\gamma(x)\geq\delta>0$ $\rho_\calX$-a.s. Then, 
    $\delta v(\ghat(x),\rho(\cdot|x))<\zeta(\delta), 
        ~\rho_\calX\text{-a.s.}\implies \ERL(d\circ\ghat) = \ERL(\fstar)$.
\end{itemize}
\end{theorem}
The first part of \cref{th:calibrationriskslownoise} is a generalization of Thm. 10 by \cite{bartlett2006convexity} to general discrete losses.
Note that combining \cref{eq:improved_calibration} with the lower bound given by \cref{th:lowerboundcalibration} gives $\bar{\zeta}^{(p)}\gtrsim\varepsilon^{(\frac{p+2}{p+1})}$. Indeed, $p=0$ corresponds to no assumption at all and so $\zeta^{(0)}$ stays quadratic, while $p\to\infty$ corresponds to having less and less noise at the boundary decision and $\zeta^{(p)}$ tends to be linear. Note that $p=\infty$ corresponds to having $\delta>0$ such that $\gamma(x)\geq\delta>0$ $\rho_{\calX}$-a.s, and so from the second part of \cref{th:calibrationriskslownoise}, one obtains zero excess risk if the excess surrogate Bayes risk is smaller than $\zeta(\delta)$ almost surely. This fact has been used in binary classification together with high probability bounds on the estimator to obtain exponential rates of convergence for the risk $\ERL$ \cite{audibert2007fast, koltchinskii2005exponential, pillaud2017exponential}, and our result could be used in the same way for the structured case. Finally, note that \cref{th:calibrationrisks} and \cref{th:calibrationriskslownoise} are not specific to $\phi$-calibrated surrogates and apply to any surrogate method. 

\subsection{Minimizing the Surrogate Loss with Averaged
Stochastic Gradient Descent (ASGD)}\label{sec:ASGD}
In this section, for simplicity, we assume $S:\calV\times\calY\rightarrow\Rspace{}$ is a loss of Legendre-type (see \cref{eq:legendretypeloss}) with associated Legendre-type potential $h:\calD\subseteq\calH\rightarrow\Rspace{}$ and $\calV=\operatorname{dom}(h^*)=\Rspace{k'}$.
Following \cite{osokin2017structured}, we provide a statistical analysis of the minimization of the expected risk of $S$ using online projected averaged stochastic gradient descent (ASGD) \cite{nemirovski2009robust} on a reproducing kernel Hilbert space (RKHS) \footnote{Recall that a scalar RKHS $\bar{\calG}$ is a Hilbert space of functions from $\calX$ to $\Rspace{}$ with an associated kernel $k:\calX\times\calX\rightarrow\Rspace{}$ such that $k(x,\cdot)\in\bar{\calG}$ for all $x\in\calX$ and $g(x)=\langle g,k(x,\cdot)\rangle_{\bar{\calG}}$ for all $g\in\bar{\calG}$.} \cite{aronszajn1950theory} $\calG=(\bar{\calG})^{\otimes k'}=\calV\otimes\bar{\calG}$, where $\bar{\calG}$ is a scalar RKHS. This will give us insight on the important quantities for the design of the surrogate method when minimizing a discrete loss $L$. Let $k:\calX\times\calX\rightarrow\Rspace{}$ be the kernel associated to the RKHS~$\bar{\calG}$ and $\sup_{x\in\calX}k(x,x)\leq\kappa^2$. 
The $n$-th update of ASGD reads
\vspace{-1pt}
\begin{equation}
   \ghat_n =\textstyle \frac{1}{n}\sum_{i=1}^ng_i, \quad g_i = \Pi_D\left(g_{i-1} - \eta_i\nabla S(g_{i-1}(x_i), y_i)\otimes k(x_i,\cdot)\right),
\end{equation}
where $\nabla S$ denotes the gradient of $S$ w.r.t the first coordinate, $\eta_i$ is the step size and $\Pi_D$ is the projection onto the ball of radius $D$ w.r.t the norm induced by $\calG$. We have the following theorem.
\begin{theorem} \label[theorem]{th:ASGD}
Let $S:\calV\times\calY\rightarrow\Rspace{}$ be a loss of Legendre-type with associated $({1}/{\beta_{\|\cdot\|_2}})$-strongly convex $h$. Let $(x_i,y_i)_{i=1}^n$, with $n \in \N$ be independently and identically distributed according to $\rho$ and
assume $\gstar \in \calG$ and $\|\gstar\|_{\calG}^2=\sum_{j=1}^{k'}\|\gstar_j\|_{\bar{\calG}}^2\leq D^2$. 

Let $C^2= 1 + \sup_{y\in\calY}\|\phi(y)-\nabla h^*(0)\|_2/(\kappa\beta_{\|\cdot\|_2}D)$, then, by using the constant step size~$\eta=2/(\beta_{\|\cdot\|_2}\kappa^2C^2\sqrt{n})$, we have
\vspace{-1pt}
\begin{equation}\label{eq:ASGDbound}
    \mathbb{E}\left[\ERL(d\circ\ghat_n) - \ERL(\fstar)\right] \leq \frac{4\cdot\kappa\cdot c_{\psi,\|\cdot\|_2}\cdot\beta_{\|\cdot\|_2}\cdot D\cdot C}{n^{1/4}}.
\vspace{-1pt}
\end{equation}
\end{theorem}
\begin{proof}[Proof]
    Let's first compute a uniform bound on the gradients as
    \begin{align}
        \|\nabla S(g(x),y)\otimes k(x,\cdot)\|_{\calG} &\leq \kappa\|\nabla S(g(x),y)\|_2 \\ 
        &= \kappa\|\nabla h^*(g(x))-\phi(y)\|_2 \\ 
        & \leq \kappa (\|\nabla h^*(g(x))-\nabla h^*(0)\|_2 + \|\phi(y)-\nabla h^*(0)\|_2) \\
        & \leq \kappa(\kappa\beta_{\|\cdot\|_2}D + \sup_{y\in\calY}\|\phi(y)-\nabla h^*(0)\|_2)=M,
    \end{align}
    where at the first step we have used that $\|k(x,\cdot)\|_{\bar{\calG}}\leq \kappa$ and at the last step that $h^*$ is $\beta_{\|\cdot\|_2}$-smooth because $h$ is $(1/\beta_{\|\cdot\|_2})$-strongly convex, $\|g(x)\|_2^2=\sum_{j=1}^{k'}\langle g_j, k(x,\cdot)\rangle_{\bar{\calG}}^2\leq\sum_{j=1}^{k'}\kappa^2 \|g_j\|_{\bar{\calG}}^2=\kappa^2 \|g\|_{\calG}^2\leq\kappa^2 D^2$ and that $\nabla h^*(v)-\nabla h^*(0)$ vanishes at the origin.
    Using classical results on ASGD \cite{nemirovski2009robust}, we know that using the constant step size $\eta=2D/(M\sqrt{n})$, we have that $\mathbb{E}[\ERS(\widehat{g}) - \ERS(\gstar)] \leq  {2DM}/{n^{1/2}}$ after $n$ iterations of ASGD. Finally, applying the lower bound on $\zeta$ in \cref{th:lowerboundcalibration}, re-arranging terms, and using the fact that $\mathbb{E}[w] \leq \sqrt{\mathbb{E}[w^2]}$, for 
    $w=\ERL(d\circ\ghat_n) - \ERL(\fstar) \geq 0$, we obtain the bound~(\ref{eq:ASGDbound}).
\end{proof}

Note that \eqref{eq:ASGDbound} is upper bounded by
$8\kappa^{1/2}c_{\psi,\|\cdot\|_{2}}\max(\kappa^{1/2}\beta_{\|\cdot\|_2}D, (c_{\phi,h}\beta_{\|\cdot\|_2}D)^{1/2})/n^{1/4}$, where $c_{\phi,h}=\sup_{y\in\calY}\|\phi(y)-\nabla h^*(0)\|_2$.
There are essentially 4 quantities appearing in the bound (\ref{eq:ASGDbound}):
$c_{\psi,\|\cdot\|_{*}}$ that depends on $L$, $c_{\phi,h}$ that bounds the marginal polytope centered at $\nabla h^*(0)$, $\beta_{\|\cdot\|}$ that depends on $h$ and $D$, which is an upper bound on the norm of the optimum
$\|\gstar\|_{\calG}=\|\nabla h(\mustar)\|_{\calG}$ which depends on the link, in this case $\nabla h$, the RKHS $\calG$, and $\mustar$. Note that the image of $\mustar$ lies in the marginal polytope, which is bounded and potentially non full-dimensional in $\calH$, so if one directly estimates $\mustar$, the hypothesis space $\calG$ has to model this constraint. The role of the link function is to remove this additional complexity from $\calG$ by mapping the marginal polytope (or a superset $\calD$ of it) to the vector space $\calV$, and consequently smoothing out $\mustar$ close to the boundary of $\calM$, leading to a smaller $\|\gstar\|_{\calG}$. The types of surrogates that directly estimate $\mustar$ are of quadratic-type (see \cref{sec:QSandCRFs}), which have $\calD=\calH$ and the link is the identity. In this case, $\calG$ has to be able to model the fact that $\operatorname{Im}(\mustar)\subseteq\calM$. The second types are of logistic-type (see CRFs in \cref{sec:QSandCRFs}), which have $\calD=\calM$, and as
$\lim_{\mu\to\partial\calM}\|\nabla h(\mu)\|=+\infty$, the link smooths out $\mustar$ close to the boundary. In this case, $\|\gstar\|_{\calG}$ can potentially be much smaller as~$\calG$ does not have to model the polytope constraint. This generalizes the idea that the logistic link is preferable for estimating class-conditional probabilities, for instance, when using linear hypothesis spaces. In between the two, there are methods with  bounded $\calD$ but different than $\calM$, such as one-vs-all methods in multiclass classification, where $\calM=\Delta_k\subsetneq\calD\subsetneq\calH$ (see \cref{app:multiclassclassification}).

\section{Analysis of Existing Surrogate Methods}\label{sec:examples}
In this section we apply the theory developed so far to derive new results on multiple surrogate methods used in practice. In \cref{sec:QSandCRFs}, we study two generic methods for structured prediction, namely, the quadratic surrogate \cite{ciliberto2016consistent, nowak2018sharp,ciliberto2018localized} and conditional random fields (CRFs) \cite{Lafferty:2001:CRF:645530.655813, settles2004biomedical}. Then, in \cref{sec:specificproblems}, we present new theoretical results on multiple tasks in supervised learning which can be derived using results from \cref{sec:theoreticalanalysis}. 
The proofs of the results and further details can be found in \cref{app:genericmethods} for \cref{sec:QSandCRFs}, and from \cref{app:binaryclassification} to \cref{app:graphmatching} for \cref{sec:specificproblems}. 
\subsection{Optimizing generic losses: Quadratic surrogate vs. CRFs}
\label{sec:QSandCRFs}
\paragraph{Quadratic Surrogate for Structured Prediction.}
The quadratic surrogate for structured prediction \cite{ciliberto2016consistent, ciliberto2018localized} has the form 
\begin{equation}
    \calV=\calH, \quad S(v,y) = \frac{1}{2}\|v - \phi(y)\|_2^2.
\end{equation}
This is a loss of Legendre-type with $\calD=\calH$, $h(u) = \frac{1}{2}\|u\|_2^2$ and $t^{-1}(u)=\nabla h^*(u)=u$.
We can exactly compute the calibration function when $\calM$ is full-dimensional. \cref{th:exactcalibrationquadraticsimple} is a simpler version of the result which holds for $\varepsilon$ small enough, the complete result can be found in \cref{app:quadraticsurrogate}.
\begin{theorem}\label[theorem]{th:exactcalibrationquadraticsimple}
If $\calM$ is full-dimensional, there exists $\varepsilon_0>0$ such that
\vspace{-5pt}
\begin{equation}
    \zeta_h(\varepsilon) = \frac{\varepsilon^2}{2\max_{(z,z')\in A}\|\psi(z)-\psi(z')\|_2^2}, \quad \forall\varepsilon\leq\varepsilon_0,
    \vspace{-5pt}
\end{equation}
where $A=\{(z,z')\in\calZ^2~|~z'\neq z, \calH_0(z)\cap\calH_0(z')\neq\varnothing\}$.
\end{theorem}
Note that in this case as $h$ is $1$-strongly convex with respect to the Euclidean norm, \cref{th:lowerboundcalibration} gives  $\zeta_h(\varepsilon) \geq \varepsilon^2\cdot(8c_{\psi, \|\cdot\|_2}^2)^{-1}$, 
and we recover the comparison inequality from \cite{ciliberto2016consistent}. In \cref{app:quadraticsurrogate} we compare this result with the lower bounds on the calibration function for the quadratic-type surrogates studied by \cite{osokin2017structured}. An interesting property of the quadratic surrogate is that one can build the estimator~$\fhat$ independently of the affine decomposition of $L$ by minimizing the expected surrogate risk with kernel ridge regression. In particular, this allows to extend the framework to continuous losses defined in compact sets $\calZ,\calY$ where $\calH$ can be infinite-dimensional \cite{ciliberto2016consistent}.

\paragraph{Conditional Random Fields.} Recall that $r=\operatorname{dim}(\calM)$. CRFs correspond to
\vspace{-2pt}
\begin{equation}\label{eq:CRFloss}
    \calV=\Rspace{r}, \quad S(v, y) = \textstyle \log(\sum_{y'\in\calY}\exp(\langle v, \phi(y')\rangle) - \langle v, \phi(y)\rangle.
\end{equation}
This is a loss of Legendre-type with $\calD=\calM$ and $h(u) = -\max_{q\in\operatorname{Prob}(\calY)}\operatorname{Ent}(q) \quad\text{s.t}\quad \mu(q)=u$,
where $\operatorname{Ent}(q)=-\sum_{y\in\calY}q(y)\log q(y)$ is the Shannon entropy of the distribution $q\in\operatorname{Prob}(\calY)$ \footnote{Note that here, for simplicity, we do not consider the constant term $h(\phi(y))$ from \cref{eq:legendretypeloss}.}.
In this case, the inverse of the link function $t^{-1}=\nabla h^*$ corresponds to performing marginal inference on the exponential family with sufficient statistics $\phi$. A well-known important drawback of CRFs is the fact that they are in general not calibrated to any specific loss. The reason being that inference in CRFs is done using MAP assignment, which corresponds to $\fhat_{\operatorname{MAP}}(x) = \argmax_{y\in\calY}~\langle\phi(y), \widehat{g}(x)\rangle$. It can be written in terms of $\nabla h^*$ as $\fhat_{\operatorname{MAP}}(x) = \phi^{-1}\left(\lim_{\gamma\to\infty}\nabla h^*(\gamma \widehat{g}(x))\right)$ (see \cref{app:CRFs} for the computation), which is different from the decoding $d_{\psi,\nabla h}$ we propose: $\fhat(x) = \argmin_{z\in\calZ}\langle\psi(z), \nabla h^*(\ghat(x))\rangle$. Note that $\nabla h^*(\gamma \widehat{g}(x))$ converges to a vertex of $\calM$ as $\gamma\to\infty$, while decoding $d_{\psi,\nabla h}$ partitions $\calM$ into $|\calZ|$ regions using $\psi$ and assigns a different output to each of those. With our proposed decoding, we can calibrate CRFs to any loss that decomposes with the cliques of the probabilistic model. For instance, for linear chains, the method consistently minimizes any loss that depends on the neighbors. Moreover, we can compute a lower bound on $\zeta_h$.
\begin{proposition}[Calibration of CRFs]\label[proposition]{prop:calibrationCRFs} The calibration function of CRFs can be lower bounded as 
$$\zeta_h(\varepsilon) \geq \frac{\varepsilon^2}{8c_{\psi,\|\cdot\|_2}^2c_{\phi,\|\cdot\|_2}^2},$$
where $c_{\phi,\|\cdot\|_2}=\sup_{y\in\calY}\|\phi(y)\|_2$.
\end{proposition}
\subsection{Specific Problems}\label{sec:specificproblems}

\paragraph{Binary Classification.}
In this case $\calZ=\calY=\{-1,1\}$. See \cref{ex:multiclassmultilabel} setting $k=2$ for     the structure of the loss.
    We consider \emph{margin losses}, which are losses of the form $S(v, y) = \Phi(yv)$ with $\calV=\Rspace{}$, where $\Phi:\Rspace{}\rightarrow\Rspace{}$ is a non-increasing function with $\Phi(0)=1$. The decoding simplifies to $d(v)=\operatorname{sign}(v)$. The logistic, exponential ($\calD=[0,1]=\Delta_2$) and square ($\calD=\Rspace{}\supsetneq\Delta_2$) margin losses are $\phi$-calibrated. The calibration function is $\zeta_h(\varepsilon) = h\left(\frac{1+\varepsilon}{2}\right)-h\left(\frac{1}{2}\right)$ \cite{bartlett2006convexity, scott2012calibrated}, where $h:\calD\subset\Rspace{}\rightarrow\Rspace{}$ is the associated potential.

\paragraph{Multiclass Classification.} In this case $\calZ=\calY=\{1,\ldots,k\}$. See \cref{ex:multiclassmultilabel} for the structure of the loss. The \emph{one-vs-all method}  corresponds to $S(v,y) = \Phi(v_y) + \sum_{j\neq y}^k\Phi(-v_j)$ with  $\calV=\Rspace{k}$. If the margin loss $\Phi(y_jv_j)$ is $\phi$-calibrated for binary classification with potential $\bar{h}:\bar{\calD}\subset\Rspace{}\rightarrow\Rspace{}$, then $S$ is $\phi$-calibrated with $h:\calD=\bar{\calD}^k\subset\Rspace{k}\rightarrow\Rspace{}$ defined as $h(u)=\sum_{j=1}^k\bar{h}(u_j)$. Note that the marginal polytope is strictly included in $\calD$: $\Delta_k\subsetneq[0,1]^k\subseteq\calD$. The decoding can be simplified to $d(v) = \argmax_{j\in[k]}~v_j$ and the calibration function has the form given by \cref{th:onevsallcalibration}.
\begin{proposition}[One-vs-all calibration function]\label[proposition]{th:onevsallcalibration}
Assume $\bar{h}''$ is non-decreasing in~$\bar{\calD}\cap[1/2,+\infty)$. Then, 
the calibration function for the one-vs-all method is $\zeta_h(\varepsilon) = 2\cdot\zeta_{\bar{h}}(\varepsilon)$.
\end{proposition}
Note that the assumption on $\bar{h}''$ is met by the logistic, exponential and square binary margin losses.
Another important example is the \emph{multinomial logistic} loss, which corresponds to \eqref{eq:CRFloss} for multiclass. In this case the decoding is also simplified to $d(v) = \argmax_{j\in[k]}~v_j$, and $\zeta_h(\varepsilon)\geq \varepsilon^2/8$ by using strong convexity of the entropy w.r.t $\|\cdot\|_1$ norm on \cref{th:lowerboundcalibration}.

\paragraph{Multilabel Classification.}
 In this case $\calZ=\calY=\{-1,1\}^k$. See \cref{ex:multiclassmultilabel} for the structure of the loss.
We consider \emph{independent classifiers}, which have the form: $S(v,y) = \sum_{j=1}^k\Phi(y_jv_j)$, with $\calV=\Rspace{k}$. In this case the potential has the form $h(u)=\sum_{j=1}^k\bar{h}((u_j+1)/2)$, where $\bar{h}$ is the potential for the individual classifier. In this case $\calM$ equals $\calD$ for logistic and exponential classifiers, as they have $\bar{\calD}=[0,1]$. The decoding is simplified to $d(v) = (\text{sign}(v_j))_{j=1}^k$ and the calibration function $\zeta_h$ can be computed exactly and it is linear in the number of labels $\zeta_h(\varepsilon)=k\cdot\zeta_{\bar{h}}(\varepsilon)$
(see \cref{prop:hammingcalibration} in \cref{app:multilabelclassification}).

\paragraph{Ordinal Regression.} In this case $\calZ=\calY=\{1,\ldots, k\}$ with an ordering: $1\prec\cdots\prec k$. We consider the absolute error loss function defined as $L(z,y) = |z-y|$. We analyze two types of surrogates, the \emph{all thresholds (AT)} \cite{lin2006large} and the \emph{cumulative link (CL)} \cite{mccullagh1980regression}, both studied by \cite{pedregosa2017consistency}.
AT methods correspond to independent classifiers $S(v,y) = \sum_{j=1}^{k-1}\Phi(\phi_j(y)v_j)$ with $\calV=\Rspace{k-1}$, where $\phi(y)=(2\cdot 1(y_j\geq j)-1)_{j=1}^{k-1}$. In this case, using results from the Hamming loss we show that 
$\zeta_h(\varepsilon)\geq (k-1)\cdot\zeta_{\bar{h}}(\varepsilon/(k-1))\sim\varepsilon^2/(k-1)$, where $\bar{h}$ is the potential for the individual classifier. CL methods, instead, are based on applying a link to the cumulative probabilities, and it can be shown that the associated potential is the negative entropy on the simplex in $k$ dimensions. Using the strong convexity of the entropy w.r.t the $\|\cdot\|_1$ norm, we can show that
$\zeta_h(\varepsilon)\geq 1/8\cdot \varepsilon^2/(k-1)^2\sim\varepsilon^2/(k-1)^2$. In particular, this explains the experiment of Fig. 1 from \cite{pedregosa2017consistency}, where they provide empirical evidence that the calibration function for AT is larger than the one for CL, which they are not able to compute.

\paragraph{Ranking with NDCG loss.} We provide guarantees for learning permutations with the NDCG loss. In this case, we recover the results from \cite{ravikumar2011ndcg}.

\paragraph{Graph matching.} In graph matching, the goal is to map the nodes from one graph to another. In this case, the outputs are ``matchings'' (permutations) and the goal is to minimize the Hamming loss between permutations $L(\sigma, \sigma') = \frac{1}{m}\sum_{j=1}^m1(\sigma(j)\neq\sigma'(j)) = 1 - {\langle X_{\sigma},X_{\sigma'}\rangle_{F}}/{m}$,
where $X_{\sigma}\in\Rspace{m\times m}$ is the permutation matrix associated to the permutation $\sigma$. In this case, $k=m^2$,
the marginal polytope is the polytope of doubly stochastic matrices which has dimension $r=\operatorname{dim}(\calM)=k^2-2k+1$, and the decoding mapping corresponds to perform linear assignment. As CRFs are intractable in this case \cite{petterson2009exponential}, a common approach is to learn the probabilities of each row independently by casting this problem as $k$ multiclass problems (one for each row of the matrix). Doing multinomial logistic regression independently at each row corresponds to set $\calD$ as the polytope of row-stochastic matrices, which strictly includes $\calM$ and
has dimension $k^2-k$. A direct application of  \cref{th:lowerboundcalibration} gives $\zeta_h(\varepsilon) \geq {m^2\varepsilon^2}/{8}$.

\subsection*{Acknowledgements}
The first author was supported by La Caixa Fellowship. This work has been possible thanks to the funding from the European Research Council (project Sequoia 724063).
We also thank David L\'opez-Paz for his valuable feedback.

\bibliography{biblio}
\newpage
\appendix

{\Huge{Organization of the Appendix}}

\begin{itemize}
\item [\textbf{A.}] \emph{\Large{Bregman Divergence Representation of Surrogate Losses}}
    \item [\textbf{B.}] \emph{\Large{Functions of Legendre-type and Canonical Link}}
    \item [\textbf{C.}] \emph{\Large{Calibration of Risks}}
    \begin{itemize}
        \item [\textbf{C.1. }]\emph{Calibration of Risks without Noise Assumption}
        \item [\textbf{C.2. }]\emph{Calibration of Risks with Noise Assumption}
    \end{itemize}
    \item [\textbf{D.}] \emph{\Large{General Results for the Calibration Function}}
    \begin{itemize}
        \item [\textbf{D.1. }]\emph{Exact Formula for the Calibration Function}
        \item [\textbf{D.2. }]\emph{Lower Bounds on the Calibration Function}
        \item [\textbf{D.3. }]\emph{Upper Bound on the Calibration Function}
    \end{itemize}
    \item [\textbf{E.}] \emph{\Large{Generic Methods for Structured Prediction}}
    \begin{itemize}
        \item [\textbf{E.1. }]\emph{Quadratic Surrogate} 
        \item [\textbf{E.2. }]\emph{Conditional Random Fields}
    \end{itemize}
    \item [\textbf{F.}] \emph{\Large{Binary Classification}}
    \item [\textbf{G.}] \emph{\Large{Multiclass Classification}}
    \begin{itemize}
        \item [\textbf{G.1. }]\emph{One-vs-All} 
        \item [\textbf{G.2. }]\emph{Multinomial Logistic}
    \end{itemize}
    \item [\textbf{H.}] \emph{\Large{Multilabel Classification with Hamming Loss}}
    \item [\textbf{I.}] \emph{\Large{Ordinal Regression}}
    \begin{itemize}
        \item [\textbf{I.1. }]\emph{All Thresholds}
        \item [\textbf{I.2. }]\emph{Cumulative Link} 
    \end{itemize}
    \item [\textbf{J.}] \emph{\Large{Ranking with NDCG Measure}}
    \item [\textbf{K.}] \emph{\Large{Graph Matching}}
\end{itemize}

\newpage 

\begin{figure}[ht!]
    \centering
    \includegraphics[width=0.9\textwidth]{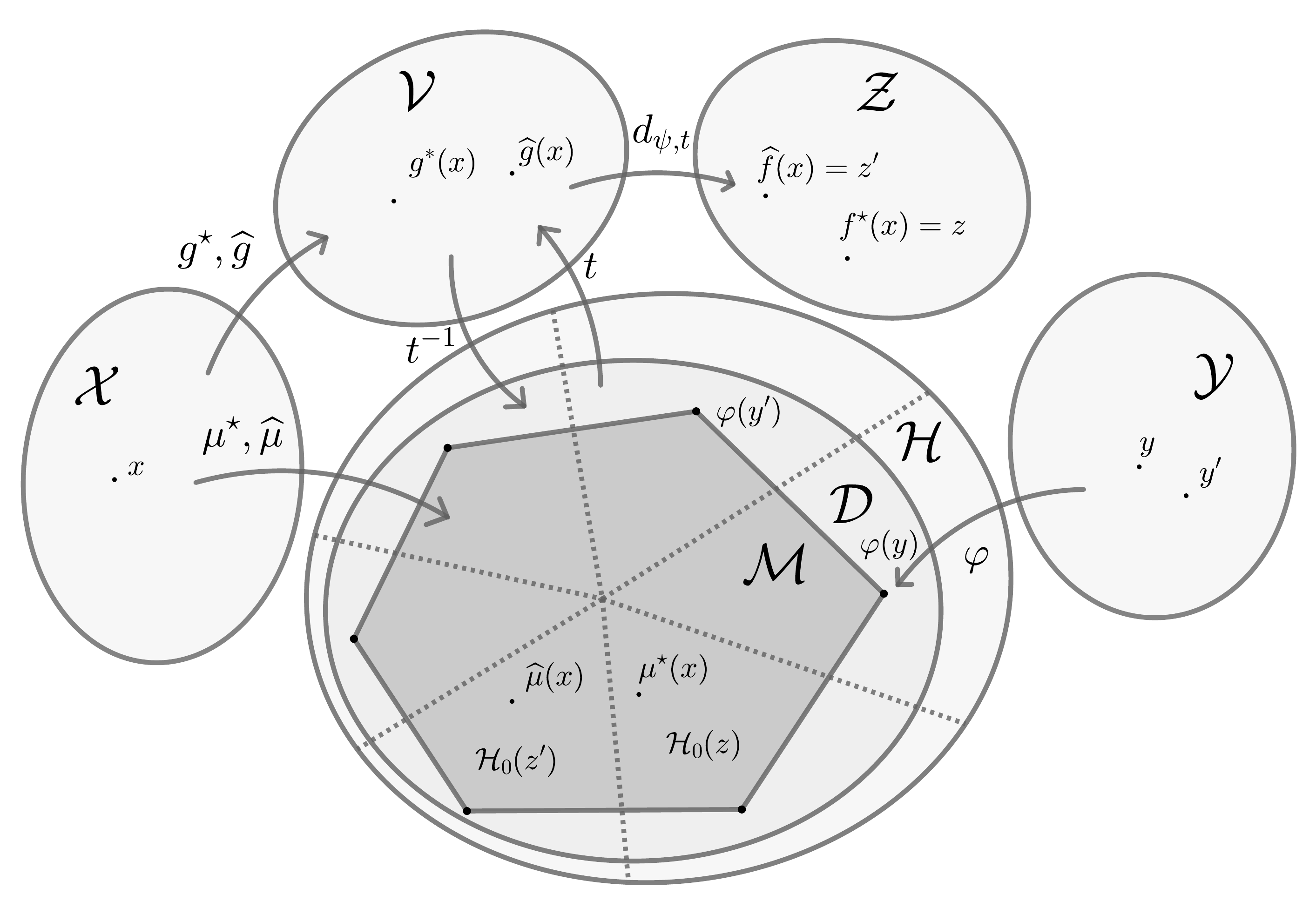}
    \caption{Diagram representing the surrogate method $(S,d_{\psi,t})$. The dashed lines separate the calibration sets $\calH_0(z)$. If $\muhat(x)=t^{-1}(\ghat(x))\in\calH_0(z')$, then $\fhat(x)=d_{\psi,t}\circ\ghat(x)=z'$.}
    \label{fig:diagram}
\end{figure}

\section{Bregman Divergence Representation of Surrogate Losses}\label{app:BDrepresentation}
\begin{proof}[Proof of \cref{th:compositerepresentation}]
As $t:\calM\rightarrow\calV$ is injective, we have that for any $v\in t(\calM)$ there exists a unique $\mu\in\calM$ such that $t(\mu)=v$. We then consider the loss $\bar{S}:\calM\times\calY\rightarrow\Rspace{}$ defined as $\bar{S}(\mu, y) = S(t(\mu), y)$. Moreover, $\bar{S}$ is a continuous function of $\mu$ because $t$ and $S$ are continuous. We define the quantities 
$\bar{s}(\mu, q)=\Expect_{Y\sim q}\bar{S}(\mu, Y)$ and 
$\delta\bar{s}(\mu, q) = \bar{s}(\mu, q) - \min_{\mu'\in\calM}\bar{s}(\mu', q)\geq 0$.

Furthermore, we have that if $S$ satisfies $t(\mu(q))=\argmin_{v\in\calV}~s(v,q)$, then $\bar{S}$ satisfies
\begin{equation}\label{eq:propertyelicitation}
    \mu(q) = \argmin_{\mu'\in\calM}~\bar{s}(\mu', q).
\end{equation}
A loss $\bar{S}$ satisfying \cref{eq:propertyelicitation} is said to elicit the function $\mu(\cdot)=\Expect_{Y\sim\cdot}\phi(Y)$. It is a known result from the theory of property elicitation \cite{abernethy2012characterization,frongillo2015vector} that
if a loss elicits a linear function of a distribution, then there exists a strictly convex function $h$ such that,
\begin{equation}
    \delta\bar{s}(\mu, q)=D_h(\mu(q), \mu), \quad  \forall q\in\operatorname{Prob}(\calY), \forall \mu\in\calM.
\end{equation}
Here, the Bregman divergence of a strictly convex (and potentially non-differentiable) function $h$ is defined as
\begin{equation}
     D_{h}(\mu', \mu) = h(\mu') - h(\mu) - \langle \mu' - \mu, dh_{\mu}\rangle,
\end{equation}
where $\{dh_{\mu}\}_{\mu\in\calM}$ are a selection of subgradients of $h$. Note that $dh_{\mu}\in T_{\mu}\calM$, where $T_{\mu}\calM=\Rspace{r}$ is the tangent space of the marginal polytope $\calM$ at the point $\mu\in\calM$, where $r=\operatorname{dim}(\calM)$.

Finally, we prove that $h$ is differentiable in $\calM$. Let's first note that as $\bar{S}(\mu, q)$ is a continuous function of $\mu$, then $\delta\bar{s}(\mu, q)=D_h(\mu(q),\mu)$ and consequently $A_{\mu'}(\mu) \defeq h(\mu) - \langle\mu-\mu', dh_{\mu}\rangle$ are also continuous in $\mu$.

Now assume that $h$ is not differentiable at $\mu_0$ and consider $dh_{\mu_0}^{(1)}$ and $dh_{\mu_0}^{(2)}$ two different subgradients of $h$ at the point $\mu_0\in\calM$. In particular, this means that there exists at least a point~$\mu_1\in\calM$ such that $\langle \mu_0-\mu_1, dh_{\mu_0}^{(1)}\rangle\neq \langle \mu_0-\mu_1, dh_{\mu_0}^{(2)}\rangle$. 

Assume without loss of generality that~$\langle \mu_0-\mu_1, dh_{\mu_0}^{(1)}\rangle < \langle \mu_0-\mu_1, dh_{\mu_0}^{(2)}\rangle$ and that \\ $\mu_0 + (\mu_0-\mu_1)\in\calM$.

Then, consider the parametrization of the segment $\operatorname{hull}(\{\mu_0 + (\mu_0-\mu_1), \mu_1\})$ as 
\begin{equation}
    \mu(t)=t\cdot \mu_1 + (1-t)\cdot (\mu_0 + (\mu_0-\mu_1)).
\end{equation}
We have that
\begin{align*}
    &\lim_{t\to(1/2)^{-}} A_{\mu_1}(\mu(t)) \\
    &\leq h(\mu_0) - \langle\mu_0-\mu_1, dh_{\mu_0}^{(1)}\rangle \\
    &< h(\mu_0) - \langle\mu_0-\mu_1, dh_{\mu_0}^{(2)}\rangle \\
    &\leq \lim_{t\to(1/2)^{+}} A_{\mu_1}(\mu(t)).
\end{align*}

Hence, $A_{\mu_1}(\mu(\cdot))$ is not continuous at $t=1/2$, which means that $A_{\mu_1}(\cdot)$ is not continuous at~$\mu_0$, which is a contradiction.
\end{proof}
\section{Functions of Legendre-type and Canonical Link}
\label{app:canonicallink}

In this section, we first introduce functions of Legendre-type and provide some of the most representative examples, namely, the quadratic function and negative maximum-entropy. We then show how Fenchel duality applied to this group of functions can be used to construct $\phi$-calibrated surrogates by taking the gradient as the link function. In particular, we show that the surrogate loss resulting from this construction, which we refer to as Legendre-type loss function, has desirable properties such as convexity and  the fact that the surrogate excess Bayes risk can be written as a Bregman divergence directly at the surrogate space $\calV$.

First, we recall the concept of \emph{essentially smooth} functions (see \cite{rockafellar2015convex} for more details).

\begin{definition}[Essentially Smooth Functions]\label[definition]{def:essentiallysmooth}
A function $h:\calD\subseteq\calH\rightarrow\Rspace{}$ is called essentially smooth if
\begin{itemize}
    \item[(1)] $\calD$ is non-empty,
    \item[(2)] $h$ is differentiable throughout $\interior(\calD)$,
    \item[(3)] and $\lim_{u\to\partial\calD}\nabla h(u)=+\infty$, where $\partial\calD$ is the boundary of the set $\calD$.
\end{itemize}
\end{definition}
\begin{definition}[Legendre-type Functions]\label[definition]{def:legendretype}
A function $h:\calD\subseteq\calH\rightarrow\Rspace{}$ is of Legendre-type if
it is strictly convex in $\interior(\calD)$ and essentially smooth.
\end{definition}
The two most important examples of such functions are the quadratic loss $h(u)=\frac{1}{2}\|u\|_2^2$ with domain $\calD=\calH\supsetneq\calM$ and the negative maximum-entropy 
\begin{equation*}
    h(u) = -\max_{q\in\operatorname{Prob}(\calY)}\operatorname{Ent}(q) \quad\text{s.t}\quad \mu(q)=u,
\end{equation*}
with $\calD=\calM\subsetneq\calH$, where $\operatorname{Ent}(q)=-\sum_{y\in\calY}q(y)\log q(y)$ is the Shannon entropy of the distribution $q\in\operatorname{Prob}(\calY)$. Now, recall the concept of \emph{Fenchel conjugate} of a function $h:\calD\subseteq\calH\rightarrow\Rspace{}$, which is defined as the function $h^*$ computed as
\begin{equation}
    h^*(v) = \sup_{u\in\calD}\{\langle v, u\rangle - h(u)\},
\end{equation}
with domain $\operatorname{dom}(h^*)$. The following \cref{prop:legendrefenchel} states that Legendre-type functions behave specially well with Fenchel duality.
\begin{proposition}[Fenchel conjugate of a Legendre-type function \cite{rockafellar2015convex}]
\label[proposition]{prop:legendrefenchel}
The Fenchel conjugate $h^*$ of a Legendre-type function $h$ is also of Legendre-type, with $\interior(\operatorname{dom}(h^*))=\interior(\operatorname{Im}(\nabla h))$. Moreover, the gradient functions $\nabla h:\interior(\calD)\rightarrow\interior(\operatorname{dom}(h^*))$ and 
$\nabla h^*:\interior(\operatorname{dom}(h^*))\rightarrow\interior(\calD)$ are inverse of each other $\nabla h^* = (\nabla h)^{-1}$. Furthermore, we also have that
$$ D_h(u, u') = D_{h^*}(\nabla h(u'), \nabla h(u)).$$
\end{proposition}

The Fenchel conjugate of the quadratic function is the same quadratic function with $\operatorname{dom}(h^*)=\calD=\calH$, while the Fenchel conjugate for the negative maximum-entropy is the log-sum-exp function~$h^*(v) = \log(\sum_{y\in\calY}\exp(\langle\psi(y),v\rangle))$ with domain $\operatorname{dom}(h^*)=\Rspace{r}$ where $r=\operatorname{dim}(\calM)$.

An interesting consequence of \cref{prop:legendrefenchel} for our framework is that one has a systematic way of constructing a surrogate method from a function of Legendre-type with domain including the marginal polytope. More specifically, we define the surrogate loss associated to $h$ as the loss with $(h,\nabla h, \phi)$-BD representation, which takes the following form 
\begin{equation}\label{eq:legendretypelossapp}
    S(v,y) = D_h(\phi(y), \nabla h^*(v)) = h^*(v) + h(\phi(y)) -
    \langle\phi(y), v\rangle.
\end{equation}
Note that $S(v,y)$ is always convex and the excess conditional surrogate risk $\delta s(v, q)$ has the form of a Bregman divergence both in $\calD$ and $\operatorname{dom}(h^*)$,
\begin{equation}
    \delta s(v, q) = D_h(\mu(q),\nabla h^*(v)) = D_{h^*}(v, \nabla h(\mu(q))),
\end{equation}
where we have used the last property of \cref{prop:legendrefenchel}.
Moreover, we have that
\begin{itemize}
    \item[-]  $\calD$ is bounded if and only if $\operatorname{dom}(h^*)$ is a vector space and $h^*$ is globally Lipschitz.
    \item[-] If $h$ is $({1}/{\beta_{\|\cdot\|}})$-strongly convex w.r.t the norm $\|\cdot\|$, then $S(\cdot,y)$ is $(\beta_{\|\cdot\|})$-smooth w.r.t the dual norm $\|\cdot\|_{*}$.
\end{itemize}

The surrogate loss associated to the quadratic function is the quadratic surrogate and has the form $S(v,y) = \frac{1}{2}\|v-\phi(y)\|_2^2$ with $\calV=\calH$, while the surrogate loss associated to the entropy corresponds to conditional random fields (CRFs) and has the form $S(v,y) = \log(\sum_{y'\in\calY}\exp(\langle\psi(y'),v\rangle)) - \langle\psi(y),v\rangle$ with $\calV=\Rspace{r}$. Both surrogates are studied in detail in \cref{sec:QSandCRFs}.

\section{Calibration of Risks}\label{app:calibrationrisks}
In this section we study the implications of the calibration function $\zeta$ for relating both excess risks. In particular, we first prove in \cref{app:calibrationwithoutnoise} that a convex lower bound $\bar{\zeta}$ of $\zeta$ satisfies $\bar{\zeta}(\ERL(d\circ\ghat) - \ERL(\fstar))\leq \ERS(\ghat) - \ERS(\gstar)$, which corresponds to \cref{th:calibrationrisks}. Then, in \cref{app:calibrationwithnoise} we improve the calibration between risks by imposing a low noise assumption at the decision boundary.

\subsection{Calibration of Risks without Noise Assumption}
\label{app:calibrationwithoutnoise}

\begin{proof}[Proof of \cref{th:calibrationrisks}]
Note that by the definition of the calibration function, we have that
\begin{equation}\label{eq:calibrationbayes}
    \zeta(\delta\ell(d\circ \ghat(x),\rho(\cdot|x))) \leq \delta s(\ghat(x),\rho(\cdot|x)).
\end{equation}
The comparison between risks is then a consequence of Jensen's inequality:
\begin{align*}
    \bar{\zeta}(\ERL(d\circ \ghat) - \ERL(\fstar))
   &= \bar{\zeta}(\Expect_{X\sim\rho_{\calX}} \delta\ell(d\circ \ghat(X),\rho(\cdot|X))) && \\
   &\leq \Expect_{X\sim\rho_{\calX}} \bar{\zeta}(\delta\ell(d\circ \ghat(X),\rho(\cdot|X))) && (\text{Jensen ineq.})\\
   &\leq \Expect_{X\sim\rho_{\calX}} \zeta(\delta\ell(d\circ \ghat(X),\rho(\cdot|X))) && (\bar{\zeta}\leq\zeta) \\
   &\leq \Expect_{X\sim\rho_{\calX}} \delta s(\ghat(X),\rho(\cdot|X)) && (\cref{eq:calibrationbayes}) \\
   & = \ERS(\ghat) - \ERS(\gstar). &&
\end{align*}
\end{proof}

\subsection{Calibration of Risks with Low Noise Assumption}
\label{app:calibrationwithnoise}
In this subsection we prove \cref{th:calibrationriskslownoise} by imposing assumptions on the the behavior of the margin function $\gamma(x)=\min_{z'\neq\fstar(x)}\delta\ell(z',\rho(\cdot|x))$ under the marginal distribution of the data $\rho_{\calX}\in\operatorname{Prob}(\calX)$. Note that the p-noise condition $\rho_{\calX}(\gamma(X)\leq \varepsilon) = o(\varepsilon^p)$
is a generalization of the Tsybakov condition for binary classification \cite{tsybakov2004optimal} and of the condition by \cite{mroueh2012multiclass} for multiclass classification to general discrete losses. Indeed, for the binary 0-1 loss ($\calY=\{-1,1\}$), $\gamma(x) = |2\eta(x)-1|$ with $\eta(x)=\rho(Y=1|x)$, so we recover the classical Tsybakov condition.

We first prove \cref{lem:lp}, which states the equivalence between the p-noise condition~$\rho_{\calX}(\gamma(X)\leq\varepsilon)=o(\varepsilon^p)$ and $1/\gamma\in L_p(\rho_{\calX})$.
\begin{lemma}\label[lemma]{lem:lp}
If the p-noise condition holds, then $1/\gamma\in L_p(\rho_{\calX})$.
\end{lemma}
\begin{proof}[Proof]
\begin{equation}
     \|1/\gamma\|_{L_p(\rho_{\calX})}^p = \Expect1/\gamma(X)^{p} = \int_{0}^\infty  pt^{p-1}\rho_{\calX}(1/\gamma(X)>t)dt 
    = \int_{0}^\infty pt^{p-1} \rho_{\calX}(\gamma(X)<t^{-1})dt.
\end{equation}
The integral converges if $\rho_{\calX}(\gamma(X)<t^{-1})$ decreases faster than $t^{-p}$.
\end{proof}

Let's now define the error set as $X_f=\{x\in\calX~|~f(x)\neq \fstar(x)\}$. 
The following \cref{lem:tsybakov_lemma}, which bounds the probability of error by a power of the excess risk, is a generalization of the Tsybakov Lemma \cite[Prop.1]{tsybakov2004optimal} for general discrete losses.
\begin{lemma}[Bounding the size of the error set]\label[lemma]{lem:tsybakov_lemma}
If $1/\gamma\in L_p(\rho_{\calX})$, then
\begin{equation}
    \rho_{\calX}(X_f)\leq \gamma_p^{\frac{1}{p+1}}(\ERL(f) - \ERL(\fstar))^{\frac{p}{p+1}}.
\end{equation}
\end{lemma}
\begin{proof}[Proof]
By the definition of the margin $\gamma(x)$, we have that:
\begin{equation}
    1(f(x)\neq \fstar(x))\leq 1/\gamma(x)\delta \ell(f(x), \rho(\cdot|x)).
\end{equation}
By taking the $\big(\frac{p}{p+1}\big)$-th power on both sides, taking the expectation w.r.t $\rho_{\calX}$ and finally applying H\"older's inequality, we obtain the desired result.
\end{proof}

We will need the following useful \cref{lem:lemma_comparison_ineq} of convex functions.
\begin{lemma}[Property of convex functions]
\label[lemma]{lem:lemma_comparison_ineq}
    Suppose $\bar{\zeta}:\Rspace{}\rightarrow\Rspace{}$ is convex and 
    $\bar{\zeta}(0)=0$. Then, for all $y'>0$, $0\leq y\leq y'$, 
    \begin{equation}
        \bar{\zeta}(y)\leq \frac{y}{y'}\bar{\zeta}(y')  \hspace{0.3cm}
        \text{and} \hspace{0.3cm} \bar{\zeta}(y')/y' \text{   is increasing on }(0,\infty).
    \end{equation}
\end{lemma}
\begin{proof}[Proof]
    Take $\alpha=\frac{y}{y'}\leq1$. The result follows directly by definition of convexity, as
    
    $$\bar{\zeta}(y) = \bar{\zeta}((1-\alpha)0 + \alpha y') \leq (1-\alpha)\bar{\zeta}(0) + \alpha \bar{\zeta}(y') = \frac{y}{y'}\bar{\zeta}(y').$$
    For the second part, re-arrange the terms in the above inequality.
\end{proof}

We now have the tools to prove \cref{th:calibrationriskslownoise}, which is an adaptation of the proof of Thm. 10 of \cite{bartlett2006convexity} which was specific to binary 0-1 loss.

\begin{proof}[Proof of part 1 \cref{th:calibrationriskslownoise}.]
The intuition of the proof is to split the Bayes excess risk into a part with low noise
$\delta \ell(f(x),\rho(\cdot|x))\leq t$ and a part with high noise $\delta \ell(f(x),\rho(\cdot|x))\geq t$. The first part will be controlled by the $p$-noise assumption and the second part by the convex lower bound of the calibration function $\bar{\zeta}$.
\begin{align*}
\ERL(d\circ g)-\ERL(\fstar) &= \Expect_{X\sim\rho_{\calX}}\delta \ell(f(X), \rho(\cdot|X)) \\
& = \Expect_{X\sim\rho_{\calX}}\left\{1(X_f)\cdot\delta\ell(f(X), \rho(\cdot|X))\right\} \\
& = \Expect_{X\sim\rho_{\calX}}\left\{\delta \ell(f(X), \rho(\cdot|X))\cdot1(X_f \cap \{\delta \ell(f(X),\rho(\cdot|X))\leq t\}\right\} \\
&+ \Expect_{X\sim\rho_{\calX}}\left\{\delta \ell(f(X), \rho(\cdot|X))\cdot 1(X_f \cap \{\delta \ell(f(X),\rho(\cdot|X))\geq t\}\right\} \\
&= A + B.
\end{align*}
\begin{itemize}
    \item \emph{Bounding the error in the region with low noise $A$:}
    \begin{equation}
        A \leq t\rho_{\calX}(X_f)
        \leq t\gamma_p^{\frac{1}{p+1}}\left(\ERL(d\circ g)-\ERL(f^*)\right)^{\frac{p}{p+1}},
    \end{equation}
    where in the last inequality we have used 
    \cref{lem:tsybakov_lemma}.
    \item \emph{Bounding the error in the region with high noise $B$:}
    
        We have that
        \begin{equation}\label{eq:bounding_B}
             \delta \ell(f(x), \rho(\cdot|x))\cdot 1(\delta \ell(f(x), \rho(\cdot|x))\geq t)
            \leq \frac{t}{\bar{\zeta}(t)}
            \bar{\zeta}
            (\delta\ell(f(x), \rho(\cdot|x))).
        \end{equation}
        In the case $\delta\ell(f(x), \rho(\cdot|x))<t$, inequality in \cref{eq:bounding_B} follows
        from the fact that $\bar{\zeta}$ is nonnegative.
        For the case $\delta\ell(f(x), \rho(\cdot|x))>t$, apply
        \cref{lem:lemma_comparison_ineq}
        with $y'= \delta\ell(f(x), \rho(\cdot|x))$ and~$y=t$.

        From \cref{th:calibrationbayes}, we have that 
        $\Expect_{X\sim\rho_{\calX}}\{1(X_f)\cdot \bar{\zeta} (\delta\ell(f(X), \rho(\cdot|X)))\} \leq \calR(g) - \calR(g^*)$. Hence, 
        \begin{equation} B\leq \frac{t}{\bar{\zeta}(t)}
            (\ERS(g) - \ERS(g^*)).
        \end{equation}
\end{itemize}

Putting everything together,
\begin{equation}
\ERL(d\circ g)-\ERL(f^*) \leq
t\gamma_p^{\frac{1}{p+1}}\left(\ERL(d\circ g)-\ERL(f^*)\right)^{\frac{p}{p+1}} +
\frac{t}{\bar{\zeta}(t)}
            (\ERS(g) - \ERS(g^*)),
\end{equation}
and hence,
\begin{equation}
    \left(\frac{\ERL(d\circ g) - \ERL(f^*)}{t}
    -\gamma_p^{\frac{1}{p+1}} \left(\ERL(d\circ g)-\ERL(f^*)\right)^{\frac{p}{p+1}} \right)
    \bar{\zeta}(t)
    \leq \ERS(g) - \ERS(g^*).
\end{equation}
Choosing $t = \frac{1}{2}\gamma_p^{\frac{-1}{p+1}}
\left(\ERL(d\circ g)-\ERL(f^*)\right)^{\frac{1}{p+1}}$
and substituting finally gives \cref{eq:improved_calibration}. 

The fact that $\bar{\zeta}^{(p)}$ never provides a worse rate than $\bar{\zeta}$ is because we have
\begin{equation}\label{eq:xd}
    \varepsilon^{\frac{p}{p+1}}\cdot\bar{\zeta}((\gamma_p^{-1}\varepsilon)^{\frac{1}{p+1}}/2)\geq \bar{\zeta}(\varepsilon \cdot(\gamma_p^{-1})^{\frac{1}{p+1}}/2).
\end{equation}
To see this, re-arrange the terms in \cref{eq:xd} to,
\begin{equation}\label{eq:x2}
    \frac{\bar{\zeta}((\gamma_p^{-1}\varepsilon)^{\frac{1}{p+1}}/2)}{\varepsilon^{\frac{1}{p+1}}} \geq
    \frac{\bar{\zeta}(\varepsilon \cdot(\gamma_p^{-1})^{\frac{1}{p+1}}/2)}{\varepsilon}.
\end{equation}
Then, \cref{eq:x2} follows from the fact that $\frac{\bar{\zeta}(t)}{t}$ is non-decreasing by \cref{lem:lemma_comparison_ineq}.
\end{proof}

\begin{proof}[Proof of part 2 of \cref{th:calibrationriskslownoise}]
If $\delta s(\gstar(x),\rho(\cdot|x))\leq\zeta(\delta)$ $\rho_{\calX}$-a.s implies by the definition of the calibration function that $\delta\ell(\fhat(x),\rho(\cdot|x))<\delta$ $\rho_{\calX}$-a.s. 

As $\gamma(x) = \min_{z'\neq f^*(x)} \delta\ell(z', \rho(\cdot|x))\geq\delta>0$ $\rho_{\calX}$-a.s., then we necessarily have $\fhat(x)=\fstar(x)$
~$\rho_{\calX}$-a.s., which implies that $\ERS(\fhat)=\ERS(\fstar)$.
\end{proof}

\section{Calibration Function for $\phi$-Calibrated Losses}
In this section we study the calibration function for $\phi$-calibrated losses. 
In \cref{app:exactcalibration}, we compute exact expressions for $\zeta$, in 
\cref{app:lowerboundcalibration} we provide lower bounds and in \cref{app:upperboundcalibration} we prove the existence of an upper bound.
\subsection{Exact Formula for the Calibration Function}
\label{app:exactcalibration}

This subsection contains three results. The first is \cref{lem:calibrationfunctionform}, which re-writes the calibration function from its definition by leveraging the BD representation of the surrogate loss. In particular, it shows that $\zeta$ only depends on $S$ through the potential $h$, and it can be written as a (constrained) minimization problem where the Bregman divergence associated to $h$ is minimized.

In the following we will denote 
\begin{equation}\label{eq:optimalprediction}
    z(u) = \argmin_{z'\in\calZ}~\langle\psi(z'), u\rangle.
\end{equation}

\begin{lemma}\label[lemma]{lem:calibrationfunctionform}
The calibration function for a $\phi$-calibrated surrogate can be written as
 \begin{equation}\label{eq:calibrationfunctionform}
    \zeta_{h}(\varepsilon) =
        \inf_{u'\in\calM, u\in\calD}D_h(u',u)\quad\text{s.t}\quad\langle\psi(z(u))-\psi(z(u')), u'\rangle\geq\varepsilon.
\end{equation}
\end{lemma}
\begin{proof}[Proof]
As $S$ is $\phi$-calibrated, we can write 
\begin{equation}
    \delta s(v,q) = D_h(\mu(q), t^{-1}(v)), \quad \forall v\in\calV, ~\forall q\in\operatorname{Prob}(\calY).
\end{equation}
Using the affine decomposition of the loss $L$, we can write
\begin{equation}
    \delta\ell(z,q)=\langle\psi(z) - \psi(z(\mu(q))), \mu(q)\rangle, \quad \forall z\in\calZ,~\forall q\in\operatorname{Prob}(\calY).
\end{equation}

Hence, the constrained minimization problem only depends on $q$ through $\mu(q)$, so the minimization over $\operatorname{Prob}(\calY)$ can be done over $\calM$. Moreover, we can write $\delta\ell(d(v),q) = \delta\ell(z(t^{-1}(v)),q)$ and $t^{-1}(\calV)=\calD$.
Hence, applying the inverse of the link to the problem one obtains
\eqref{eq:calibrationfunctionform}.
\end{proof}

The second result is \cref{th:exactcalibration}, which uses the result of \cref{lem:calibrationfunctionform} to view the problem \eqref{eq:calibrationfunctionform} as the minimum over $z$ of the Bregman divergence between the sets $\calH_{\varepsilon}(z)^c\cap\calM$ and $\calH_{0}(z)\cap\calD$.
\vspace{1pt}

\begin{proof}[Proof of \cref{th:exactcalibration}]
    Use the fact that $\calH=\bigcup_{z\in\calZ}\calH_0(z)$ with $z=z(u)\iff u\in\calH_0(z)$ to re-write the calibration function as
    \begin{equation}
        \zeta_{h}(\varepsilon) = \min_{z\in\calZ}
            \inf_{\substack{u'\in\calM \\ u\in \calH_0(z)\cap\calD}}D_h(u',u)\quad\text{s.t}\quad\langle\psi(z)-\psi(z(u')), u'\rangle\geq\varepsilon.
    \end{equation}
    Now, the minimization over the first coordinate is made on the set (now independent of $u$)
    \begin{equation}
        \{u'\in\calH~|~\langle\psi(z)-\psi(z(u')), u'\rangle\geq\varepsilon\}\cap\calM = 
        \{u'\in\calH~|~\langle\psi(z)-\psi(z(u')), u'\rangle\leq\varepsilon\}^c\cap\calM .
    \end{equation}
    Now let's show that
    \begin{equation}
        \{u'\in\calH~|~\langle\psi(z)-\psi(z(u')), u'\rangle\leq\varepsilon\}=\calH_\varepsilon(z),
    \end{equation}
    where $\calH_\varepsilon(z) = \{u'\in\calH~|~\langle\psi(z)-\psi(z'), u'\rangle\leq\varepsilon, \forall z'\in\calZ\}$. Note that the inclusion 
    $(\supset)$ is trivial. For $(\subset)$, note that for any $z'\in\calZ$, we have that
    \begin{equation}
        \langle\psi(z'),u'\rangle+\varepsilon\geq\langle\psi(z(u')), u'\rangle+\varepsilon \geq \langle\psi(z), u'\rangle.
    \end{equation}
    Hence, we obtain the final result,
    \begin{equation}
        \zeta_{h}(\varepsilon) = \min_{z\in\calZ} \inf_{\substack{u'\in\calH_\varepsilon(z)^c\cap\calM \\ u\in \calH_0(z)\cap\calD}}D_h(u',u) = \min_{z\in\calZ}
        D_h(\calH_\varepsilon(z)^c\cap\calM,\calH_0(z)\cap\calD).
    \end{equation}
\end{proof}

Finally, the following \cref{prop:l2distancecalibration} provides an exact formula for the Euclidean distance between the sets $\calH_{\varepsilon}(z)^c$ and $\calH_0(z)$. This result will be useful to derive an improved lower bound on~$\zeta_h$ using strong convexity of the potential w.r.t the Euclidean norm and an exact expression in the case of the quadratic surrogate when the marginal polytope $\calM$ is full-dimensional. In the following, we denote by $d_2(A,B)=\inf_{u'\in A,u\in B}\|u'-u\|_2$ the Euclidean distance between the sets $A$ and $B$.

\begin{proposition}\label[proposition]{prop:l2distancecalibration}
We have that
\begin{equation}\label{eq:l2distancecalibration}
    d_2(\calH_{\varepsilon}(z)^c, \calH_0(z)) = \min_{z'\neq z}\frac{\varepsilon}{\|\psi(z)-\psi(z')\|_2} + \delta_{z,z'},
\end{equation}
where $\delta_{z,z'}= d_2(\calH_0(z),\{\langle\psi(z)-\psi(z'),u\rangle=0\})>0$ if and only if $\calH_0(z)\cap\calH_0(z')=\varnothing$.
\end{proposition}
\begin{proof}[Proof]
Write $\calH_{\varepsilon}(z)^c=\bigcup_{z'\neq z}\{\langle\psi(z)-\psi(z'),u\rangle\geq \varepsilon\}$. Then, we have that
\begin{equation}
    d_2(\calH_{\varepsilon}(z)^c, \calH_0(z)) =
    \min_{z'\neq z}d_2(\calH_0(z), \{u\in\calH~|~\langle\psi(z)-\psi(z'),u\rangle\geq\varepsilon\}),
\end{equation}
where we have used that $d_2(\cup_{i=1}^kA_i,B)=\min_{i=1,\ldots,k}d_2(A_i,B)$.
Recall that the distance between two convex bodies is characterized by the minimum distance between two parallel supporting hyperplanes. Let's split the analysis in two cases:
\begin{itemize}
    \item[-] $\calH_0(z)\cap\calH_0(z')\neq\varnothing$. In this case $\{\langle\psi(z)-\psi(z'),u\rangle=0\}$ is a supporting hyperplane of $\calH_0(z)$. Hence, the distance betwen $\calH_0(z)$ and $\{\langle\psi(z)-\psi(z'),u\rangle\geq\varepsilon\}$ is equal to the distance between 
    $\{\langle\psi(z)-\psi(z'),u\rangle\geq\varepsilon\}$ and $\{\langle\psi(z)-\psi(z'),u\rangle\leq 0\}$, which is equal to $\frac{\varepsilon}{\|\psi(z)-\psi(z')\|_2}$.
    
    \item[-] $\calH_0(z)\cap\calH_0(z')=\varnothing$. In this case $\{\langle\psi(z)-\psi(z'),u\rangle=0\}$ is not a supporting hyperplane of $\calH_0(z)$. The supporting hyperplane parallel to $\{\langle\psi(z)-\psi(z'),u\rangle=\varepsilon\}$ has the form
    $\{\langle\psi(z)-\psi(z'),u\rangle=-\varepsilon'\}$ for $\varepsilon'>0$. The distance between both hyperplanes is then 
    $\frac{\varepsilon + \varepsilon'}{\|\psi(z)-\psi(z')\|_2} = \frac{\varepsilon}{\|\psi(z)-\psi(z')\|_2} + \delta_{z,z'}$, 
    where $\delta_{z,z'}=d_2(\calH_0(z),\{\langle\psi(z)-\psi(z'),u\rangle=0\})$.
\end{itemize}
Hence, we obtain the final result.
\end{proof}

Note that many losses satisfy $\calH_{0}(z)\cap\calH_0(z')\neq\varnothing$ for all $z,z'\in\calZ$, such as the 0-1, Hamming and absolute error in ordinal regression. In this case, \cref{eq:l2distancecalibration} is simplified to
$\frac{\varepsilon}{\max_{z'\neq z}\|\psi(z)-\psi(z')\|_2}$.

\subsection{Lower Bounds on the Calibration Function}\label{app:lowerboundcalibration}

In this subsection, we provide two lower bounds on the calibration function. The first one is given by \cref{th:lowerboundcalibration} and allows to exploit strong convexity of the potential w.r.t an arbitrary norm in $\calH$. The second one is given by \cref{th:improvedlowerbound} and provides a tighter bound but it is special to strong convexity w.r.t the Euclidean norm.

The proof of \cref{th:lowerboundcalibration} relies on the following \cref{lem:boundexcessrisk}.

\begin{lemma}[Generic bound on the Bayes excess risk]\label[lemma]{lem:boundexcessrisk}
Let $\|\cdot\|$ be a norm in $\Rspace{k}$ and
denote by $\|\cdot\|_{*}$ its dual norm. We have that,
\begin{equation}
   \delta\ell(z(u),q)
    \leq 2~c_{\psi,\|\cdot\|_{*}}~\|u - \mu(q)\|,
\end{equation}
where $c_{\psi,\|\cdot\|_{*}}=\sup_{z\in\calZ}\|\psi(z)\|_{*}$.
\end{lemma}
\begin{proof}[Proof]
Decompose the excess Bayes risk $\delta\ell(z(u),q)$ into two terms $A$ and $B$:
\begin{align*}
     \delta\ell(z(u),q)
    &= \langle\psi(z(u)), \mu(q)-u\rangle   \\
    & + \langle\psi(z(u)), u\rangle - \langle\psi(z(\mu(q))), \mu(q)\rangle \\
    &= A + B.
\end{align*}
For the first term, we directly have 
$A\leq \sup_{z\in\calZ}|\langle\psi(z), \mu(q)-u\rangle|$.
For the second term, we use the fact that for any given two functions $\eta,\eta':\calZ\rightarrow\Rspace{}$, it holds that $|\min_z\eta(z) - \min_{z'}\eta'(z')|\leq \sup_z|\eta(z)-\eta'(z)|$. As $z(\mu(q))$ minimizes $\langle\psi(z), \mu(q)\rangle$ and $z(u)$
minimizes $\langle\psi(z), u\rangle$, we can conclude also that $B\leq \sup_{z\in\calZ}|\langle\psi(z), \mu(q)-u\rangle|$.
Hence, we obtain
\begin{equation}
    \delta(z(u),q)
    \leq 2~\sup_{z\in\calZ} ~|\langle\psi(z), \mu(q)-u\rangle|\leq 2~c_{\phi,\|\cdot\|_{*}}~\|u - \mu(q)\|,
\end{equation}
where at the last step we have used Cauchy-Schwarz inequality.
\end{proof}

We now proceed to the proof of \cref{th:lowerboundcalibration}.
\vspace{5pt}

\begin{proof}[Proof of \cref{th:lowerboundcalibration}]
    Starting from \cref{lem:calibrationfunctionform}, we see that the constrains are $\langle\psi(z(u))-\psi(z(u')), u'\rangle\geq\varepsilon$.
    From the definition of Bregman divergence and strong convexity we have that if $h$ is $({1}/{\beta_{\|\cdot\|}})$-strongly convex in $\calD$, then 
    \begin{equation}
        D_h(u', u)\geq \frac{1}{2\beta_{\|\cdot\|}}\|u'-u\|^2, \quad \forall u,u'\in\calD.
    \end{equation}
    From \cref{lem:boundexcessrisk} we have that $\langle\psi(z(u))-\psi(z(u')), u'\rangle \leq 2c_{\psi,\|\cdot\|_{*}}\|u'-u\|$.
    Putting all these things together, we obtain,
    \begin{align}
        D_h(u', u)\geq \frac{1}{2\beta_{\|\cdot\|}}\|u'-u\|^2\geq \frac{\langle\psi(z(u))-\psi(z(u')), u'\rangle^2}{8c_{\psi,\|\cdot\|_{*}}^2\beta_{\|\cdot\|}}\geq 
        \frac{\varepsilon^2}{8c_{\psi,\|\cdot\|_{*}}^2\beta_{\|\cdot\|}}.
    \end{align}
\end{proof}

Finally, we present \cref{th:improvedlowerbound}, which is based on \cref{prop:l2distancecalibration} and provides a tighter lower bound under strong convexity w.r.t the Euclidean distance.
\begin{theorem}[Improved lower bound for $L_2$-strong convexity]
\label[theorem]{th:improvedlowerbound}
If $h$ is $({1}/{\beta_{\|\cdot\|_2}})$-strongly convex w.r.t the $L_2$ norm $\|\cdot\|_2$, then:
\begin{equation}\label{eq:improvedlowerboundquadratic}
    \zeta_h(\varepsilon) \geq \frac{\varepsilon^2}{2\beta_{\|\cdot\|_2}\max_{z'\neq z}\|\psi(z)-\psi(z')\|_2^2}.
\end{equation}
\end{theorem}
\begin{proof}[Proof]
\begin{align*}
    \zeta_h(\varepsilon) & \geq \frac{1}{2\beta_{\|\cdot\|_2}}\min_{z\in\calZ}d_2(\calH_{\varepsilon}(z)^c\cap\calM, \calH_0(z)\cap\calD)^2 && (D_h(u',u)\geq\frac{1}{2\beta_{\|\cdot\|_2}}\|u'-u\|_2^2)  \\
    &\geq
    \frac{1}{2\beta_{\|\cdot\|_2}}\min_{z\in\calZ}d_2(\calH_{\varepsilon}(z)^c, \calH_0(z))^2 && \text{(minimization on larger domain)} \\
    &= \frac{1}{2\beta_{\|\cdot\|_2}}\min_{z\in\calZ}\left(\min_{z'\neq z}\frac{\varepsilon}{\|\psi(z)-\psi(z')\|_2} + \delta_{z,z'}\right)^2 && \text{(\cref{prop:l2distancecalibration})}  \\
    &\geq \frac{1}{2\beta_{\|\cdot\|_2}}\min_{z'\neq z}\frac{\varepsilon^2}{\|\psi(z)-\psi(z')\|_2^2}
    && (\delta_{z,z'}\geq 0) \\
    &= \frac{\varepsilon^2}{2\beta_{\|\cdot\|_2}\max_{z'\neq z}\|\psi(z)-\psi(z')\|_2^2}. && 
\end{align*}
\end{proof}
Note that the lower bound \eqref{eq:improvedlowerboundquadratic} is tighter than the one given by \cref{th:lowerboundcalibration}:
\begin{equation}
    \frac{\varepsilon^2}{2\beta_{\|\cdot\|_2}\max_{z'\neq z}\|\psi(z)-\psi(z')\|_2^2} \geq \frac{\varepsilon^2}{8\beta_{\|\cdot\|_2}\max_{z\in\calZ}\|\psi(z)\|_2^2},
\end{equation}
using that $\|\psi(z)-\psi(z')\|_2^2\leq \|\psi(z)\|_2^2 + \|\psi(z')\|_2^2 \leq 2\cdot \max_{z\in\calZ}\|\psi(z)\|_2^2$.
\subsection{Upper bound on the Calibration Function}
\label{app:upperboundcalibration}

In this subsection we prove the result of existence of a quadratic upper bound on $\zeta_h$. The idea of the proof is to show that there exists a point $u_0\in\calH_0\cap\relinterior(\calM)$ and a continuous path $(u_{\varepsilon})_{\varepsilon\leq\varepsilon_0}$ such that
$u_{\varepsilon}\in\calH_{\varepsilon}^c\cap\relinterior(\calM)$ with $\|u_{\varepsilon}-u_0\|\lesssim\varepsilon$ for all $\varepsilon\leq\varepsilon_0$. Then, the norm of the Hessian of $h$ can be uniformly bounded in this compact continuous path and the result follows. It is important to take this sequence at the relative interior of the marginal polytope because $\|\nabla^2h\|_2$ could explode at the boundary.

Note that if $\calM$ is full dimensional, the result follows easily from 
\cref{prop:l2distancecalibration}.
We begin by constructing this path as a segment in \cref{lem:segment}.

\begin{lemma}\label[lemma]{lem:segment}
 There exists $z\in\calZ$ and a closed segment $I=\operatorname{hull}(\{u_{\varepsilon_0}, u_0\})\subset\relinterior(\calM)$ with $u_0\in\calH_0(z)$ such that the point $u_\varepsilon = u_{0}\cdot(1-\varepsilon/\varepsilon_0) + u_{\varepsilon_0}\cdot(\varepsilon/\varepsilon_0)\in I$ satisfies for a constant $C\in\Rspace{}$:
 \begin{equation}
     u_\varepsilon\in\calH_{\varepsilon}(z)^c \hspace{1cm}\text{and}\hspace{1cm}
     \|u_{\varepsilon}-u_0\|_2 \leq C\cdot \varepsilon, \quad \forall  \varepsilon\leq\varepsilon_0.
 \end{equation}
 \end{lemma}
 \begin{proof}[Proof]
 
 We will first assume $\calM$ is \emph{non full-dimensional}. Hence, it lies in an affine subspace of $\calH$.
 Take $u_0\in\partial\calH_0(z)\cap\relinterior(\calM)$ and take $z'\in\calZ$ corresponding to a supporting hyperplane of $\calH_0(z)$ at $u_0$, i.e,
 $\langle\psi(z)-\psi(z'),u_0\rangle=0$.
 
 Using that $\calH_{\varepsilon}(z)^c=\bigcup_{z'\neq z}\{\langle\psi(z)-\psi(z'),u\rangle\geq \varepsilon\}$, we have that
\begin{align}
    &d_2(u_0, \calH_{\varepsilon}(z)^c\cap\calM)  \\
    &= \min_{z''\neq z}d_2(u_0, \calM\cap\{\langle\psi(z)-\psi(z''),u\rangle\geq\varepsilon\}) \\
    \label{eq:distanceM}
    &\leq d_2(u_0, \calM\cap\{\langle\psi(z)-\psi(z'),u\rangle\geq\varepsilon\}).
\end{align}

Now, consider a convex neighborhood $U\subset\relinterior(\calM)$ of $u_0$. We have that for $\varepsilon_0$ small enough, the distance in \cref{eq:distanceM} is achieved at $U$ at a point $u_{\varepsilon_0}\in U\subset\relinterior(\calM)$.
Moreover, we have that
\begin{equation}
    u_\varepsilon = u_{0}(1-\varepsilon/\varepsilon_0) + u_{\varepsilon_0}(\varepsilon/\varepsilon_0) = \argmin_{u\in U\cap\{\langle\psi(z)-\psi(z'),u\rangle\geq\varepsilon\}}\|u_0-u\|_2, \quad \forall \varepsilon\leq\varepsilon_0,
\end{equation}
and 
\begin{align*}
    &\|u_\varepsilon-u_0\|_2 \\
    &=d_2(u_0, U\cap\{\langle\psi(z)-\psi(z'),u\rangle\geq\varepsilon\}) \\
    &=  L\cdot d_2(u_0, \{\langle\psi(z)-\psi(z'),u\rangle\geq\varepsilon\}) \\
    &= L\cdot\frac{\varepsilon}{\|\psi(z)-\psi(z')\|_2}=C\cdot\varepsilon.
\end{align*}

For the \emph{full-dimensional} case, the proof follows the same with $L=1$.
 \end{proof}

\begin{proof}[Proof of \cref{th:upperboundcalibration}]
 We will show that for a sufficiently small $\varepsilon_0$, there exists $C'\in\Rspace{}$ such that~$\zeta_h(\varepsilon)\leq C'\cdot\varepsilon^2$ for all $\varepsilon\leq\varepsilon_0$.
 First use \cref{lem:segment} and define (using that $h$ is twice differentiable)~$C_{I}=\sup_{u'\in I}\|\nabla^2 h(u')\|_2<+\infty$ which is finite because $I=\bar{I}\subset\interior(\calM)$.
 
 Then, for all $\varepsilon\leq\varepsilon_0$, the proof follows as:
\begin{align*}
    \zeta_h(\varepsilon) &= \min_{z'\in\calZ}D_h(\calH_\varepsilon(z')^c\cap\calM,\calH_0(z')\cap\calD) \\
    &\leq \min_{z'\in\calZ}D_h(\calH_\varepsilon(z')^c\cap\calM,\calH_0(z')\cap\calM) \\
    &\leq D_h(\calH_\varepsilon(z)^c\cap\calM,\calH_0(z)\cap\calM) \\
    &\leq D_h(u_\varepsilon, u_0) \\
    &\leq C_{I}\cdot \|u_{\varepsilon}-u_0\|_2^2 \\
    &\leq  C_{I}\cdot C\cdot \varepsilon^2 = C'\cdot\varepsilon^2. 
\end{align*}
\end{proof}

\section{Generic Methods for Structured Prediction}\label{app:genericmethods}
In this section we present results on two generic methods for structured prediction: the quadratic surrogate in \cref{app:quadraticsurrogate} and conditional random fields (CRFs) in \\ \cref{app:CRFs}.
\subsection{Quadratic Surrogate}
\label{app:quadraticsurrogate}
We first provide an exact formula for the calibration function of the quadratic surrogate when the marginal polytope $\calM$ is full-dimensional. Note that in this case, one can directly apply \cref{prop:l2distancecalibration} if one makes sure that the distances are achieved inside $\calM$.

\begin{theorem}[Exact Calibration for Quadratic Surrogate]\label[theorem]{th:exactcalibrationquadratic}
Let $h:\calD=\calH\rightarrow\Rspace{}$ be $h(\cdot)=\frac{1}{2}\|\cdot\|_2^2$ corresponding to the quadratic surrogate. If $\calM$ is full-dimensional, then
\begin{equation}\label{eq:exactquadratic}
\zeta_h(\varepsilon) = \frac{1}{2}\left(\min_{z'\neq z}\frac{\varepsilon}{\|\psi(z)-\psi(z')\|_2} + \delta_{z,z'}\right)^2, \quad 
\forall\varepsilon\leq \min_{z\neq z'}\|L_z-L_{z'}\|_{\infty},
\end{equation}
where $\delta_{z,z'}= d_2(\calH_0(z),\{\langle\psi(z)-\psi(z'),u\rangle=0\})>0$ if and only if $\calH_0(z)\cap\calH_0(z')=\varnothing$, and $L_z$ is the $z$-th row of the loss matrix $L\in\Rspace{\calZ\times\calY}$.
\end{theorem}
\begin{proof}[Proof]
Note that the calibration for the quadratic surrogate is
\begin{equation}
    \zeta_h(\varepsilon) = \frac{1}{2}\min_{z'\neq z}~d_2(\calH_{\varepsilon}(z)^c\cap\calM, \calH_0(z))^2.
\end{equation}
Hence, the goal of the proof is to show that
$d_2(\calH_{\varepsilon}(z)^c, \calH_0(z)) = d_2(\calH_{\varepsilon}(z)^c\cap\calM, \calH_0(z))$ if~$\varepsilon\leq \min_{z\neq z'}\|L_z-L_{z'}\|_{\infty}$ and then use \cref{prop:l2distancecalibration}.

Remember from the proof of \cref{prop:l2distancecalibration} that the term 
$\frac{\varepsilon}{\|\psi(z)-\psi(z')\|_2}$ is the distance between the half-spaces $\{\langle\psi(z)-\psi(z'),u\rangle \leq 0\}$ and  $\{\langle\psi(z)-\psi(z'),u\rangle \geq\varepsilon\}$. This distance is achieved inside of the marginal polytope if 
$\varepsilon\leq \sup_{\mu\in\calM}|\{\langle\psi(z)-\psi(z'),\mu\rangle|$.
Hence, we have that $d_2(\calH_{\varepsilon}(z)^c\cap\calM, \calH_0(z))=d_2(\calH_{\varepsilon}(z)^c, \calH_0(z))$ if 
\begin{align*}
    \varepsilon 
    &\leq \min_{z'\neq z}\sup_{\mu\in\calM}|\{\langle\psi(z)-\psi(z'),\mu\rangle| \\
    &= \min_{z'\neq z}\sup_{q\in\operatorname{Prob}(\calY)}|\ell(z,q)-\ell(z',q)| \\
    &= \min_{z'\neq z}\sup_{q\in\operatorname{Prob}(\calY)}|(L_z-L_{z'})\cdot q|\\
    &= \min_{z\neq z'}\|L_z-L_{z'}\|_{\infty}.
\end{align*}
\end{proof}
Now we prove \cref{th:exactcalibrationquadratic}, which states that if one takes $\varepsilon$ small enough, then the expression can be simplified by removing the $\delta_{z,z'}$'s from \cref{eq:exactquadratic}.

\begin{proof}[Proof of \cref{th:exactcalibrationquadraticsimple}]
Take a 3-tuple $(z,z',z'')$ such that $\calH_0(z)\cap\calH_0(z')\neq\varnothing$
and
$\calH_0(z')\cap\calH_0(z'')=\varnothing$. Then, it is clear that
there exists $\varepsilon_0'>0$ such that for all $\varepsilon\leq\varepsilon_0'$:
\begin{equation}
    \frac{\varepsilon}{\|\psi(z)-\psi(z')\|_2} \leq \frac{\varepsilon}{\|\psi(z')-\psi(z'')\|_2} + \delta_{z',z''},
\end{equation}
because $\delta_{z',z''}>0$ as 
$\calH_0(z')\cap\calH_0(z'')=\varnothing$.
Taking $\varepsilon_0$ as the minimum of the $\varepsilon_0'$'s over all 3-tuples of this type gives the desired result.
\end{proof}

\paragraph{Kernel ridge regression as an estimator independent of the affine decomposition of $L$.}
It was shown by \cite{ciliberto2016consistent} that if one minimizes the expected risk of the quadratic surrogate using kernel ridge regression, then one can construct an estimator independent of the affine decomposition of the loss.

Indeed, given $n$ data points $\{(x_i,y_i)\}_{i\leq n}$ and a kernel $k:\calX\times\calX\rightarrow\Rspace{}$ with corresponding RKHS $\calG$, the kernel ridge regression estimator $\widehat{g}_n\in\calG\otimes\calH$ of $\gstar$ can be written in
closed form as~$\widehat{g}_n(\cdot)=\sum_{i=1}^n\alpha_i(\cdot)\psi(y_i)$ where $\alpha(x) = (\alpha_1(x),\dots,\alpha_n(x)) \in \R^n$ is defined by
$\alpha(x) = (K + n\lambda I)^{-1}K_x$ with 
$K \in \R^{n\times n}$ is defined by $K_{ij} = k(x_i, x_j)$ and $K_x = (k(x,x_1),\dots, k(x,x_n)) \in \R^n$. Note that $\widehat{g}_n$ is linear in the embeddings $\phi(y_i)_{i\leq n}$ and the link function is the identity, so the estimator is independent of the choice of the embedding of the loss because
\begin{equation}
    \fhat(x) = \argmin_{z\in\calZ}\langle\psi(z), t^{-1}(\ghat(x)\rangle =  \argmin_{z\in\calZ}\sum_{i=1}^n\alpha_i(x)L(z, y_i).
\end{equation}

In the following, we compare our calibration results on the quadratic surrogate with the work by \cite{osokin2017structured}.
\paragraph{Comparison with related work on the quadratic surrogate for structured prediction.}
In the work by \cite{osokin2017structured}, they study the calibration properties of a quadratic-type surrogate, which is constructed differently than ours. In order to understand their construction under our framework, let's consider the following decomposition of the loss function,
$L(z,y) = \langle\psi(z), \psi(y)\rangle = \langle e_z, (L\cdot e_y)\rangle$, where $L$ is the loss matrix.
Note that in this case the quadratic surrogate associated to this decomposition is $S(v,y) = \frac{1}{2}\|v - L_{:,y}\|_2^2$
with decoding $d(v) = \argmin_{z\in\calZ}~v_z$. As $v\in\Rspace{\calY}$ can be an exponentially large vector, they consider the parametrization $v=F\cdot w$, where $F\in\Rspace{\calY\times k}$ is a score matrix and $w\in\Rspace{k}$, where $k$ can be potentially much smaller than $\calY$. The surrogate loss they consider is
\begin{equation}\label{eq:osokinquadraticsurrogate}
    S(w,y) = \frac{1}{2}\|F\cdot w - L_{:,y}\|_2^2,
\end{equation}
where $L_{:,y}$ is the $y$-th column of $L$.
In their work they normalize the surrogate loss by $|\calY|$, but we remove this factor in order to properly compare calibration functions\footnote{If you multiply a surrogate by a factor, the associated calibration function gets multiplied by the same factor.}.
It is important to note that this loss does not fall into our framework for $F$ different than the identity.
They provide the following lower bound on the calibration function
\begin{equation}\label{eq:lowerboundosokin}
    \zeta(\varepsilon) \geq \frac{\varepsilon^2}{2\max_{z'\neq z}\|P_F(e_z-e_{z'})\|_2^2},
\end{equation}
where $P_F=F(F^TF)^\dagger F^T$ is the orthogonal projection to the subspace generated by the columns of $F$. In order to compare with our work, we follow \cite{nowak2018sharp} and consider a decomposition $L=F\cdot U^T$ with $F\in\Rspace{\calZ\times k},U\in\Rspace{\calY\times k}$ and $S(v,y)=\frac{1}{2}\|v-U_{y}\|_2^2$ with decoding $d(v) = \argmin_{z\in\calZ}~F_z\cdot v$, where
$F_z=\psi(z)=F^T\cdot e_z$ and $U_{y}=\phi(y)=U^T\cdot e_y$. For this surrogate method, \cref{th:improvedlowerbound} provides the following lower bound:
\begin{equation}\label{eq:lowerboundours}
    \zeta_h(\varepsilon) \geq \frac{\varepsilon^2}{2\max_{z'\neq z}\|F^T(e_z-e_{z'})\|_2^2}.
\end{equation}
Note the similarity between expressions \eqref{eq:lowerboundours} and \eqref{eq:lowerboundosokin}. In particular, if $F\in\Rspace{\calY\times\calY}$ is the identity, (so that surrogate \eqref{eq:osokinquadraticsurrogate} enters our framework), both expressions are equal. For other $F$'s, both calibration functions are not comparable since their surrogate is larger than ours. For instance, if one takes $F\in\Rspace{\calZ\times k}$ with the smallest $k$ such that $L=F\cdot U^T$ for the Hamming loss, their calibration function is proportional to $\frac{|\calY|}{k}$ \cite{osokin2017structured}, while ours is linear in $k$ (see \cref{prop:hammingcalibration}). Indeed, their surrogate is larger by construction because it is defined in $\Rspace{\calY}$, while ours is defined in $\Rspace{k}$. It is important to note that our surrogates are the ones used in practice while theirs require a summation over $|\calY|$ elements (see \eqref{eq:osokinquadraticsurrogate}), which in structured prediction is in general exponentially large.

\subsection{Conditional Random Fields} \label{app:CRFs}
This subsection has two parts. In the first one, we show how changing the decoding procedure in CRFs from MAP assignment (what it is used in practice) to the decoding we propose, it is possible to calibrate CRFs to any discrete loss with affine decomposition $L(z,y)=\langle\psi(z),\phi(y)\rangle+c$, where $\phi(y)$ are the sufficient statistics of the CRF.
At the second part, we prove the convex lower bound on the calibration function.

\begin{figure}[ht!]
    \centering
    \includegraphics[width=0.6\textwidth]{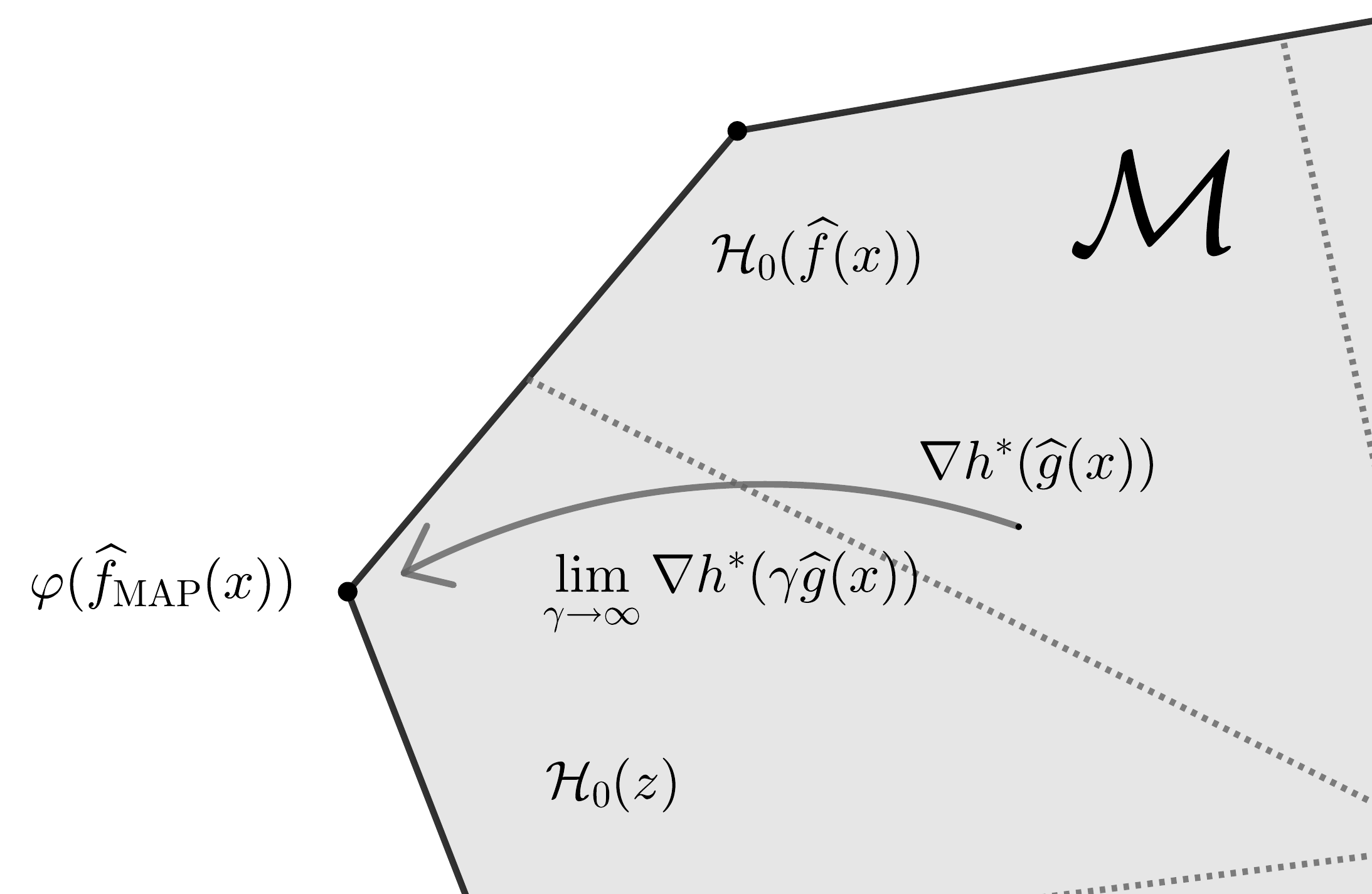}
    \caption{MAP assignment can be written as $\fhat_{\operatorname{MAP}}(x)=\phi^{-1}\left(\lim_{\gamma\to\infty}\nabla h^*(\gamma \widehat{g}(x))\right)$, which corresponds to continuously move the vector of predicted marginals $\nabla h^*( \widehat{g}(x))$ to a vertex of the marginal polytope $\lim_{\gamma\to\infty}\nabla h^*(\gamma \widehat{g}(x))$. The decoding $d_{h,\nabla h}$ corresponds to assign to $\nabla h^*(\ghat(x))$ an output element $z\in\calZ$ such that $\nabla h^*(\ghat(x))\in\calH_0(z)$.}
    \label{fig:MAP}
\end{figure}

\paragraph{On calibration of CRFs and MAP assignment. } 

Using MAP as a decoding mapping does not calibrate CRFs to any discrete loss in general. This is not the case for multinomial logistic regression (which is the equivalent method in multiclass classification), where our decoding corresponding to the multiclass 0-1 loss is exactly MAP assignment (see \cref{app:multinomiallogistic}). 

A way to understand the difference between both decoding mappings is to write MAP assignment in terms of $\nabla h^*$ as:
\begin{align*}
&\fhat_{\operatorname{MAP}}(x) = \argmax_{y\in\calY}~\langle\phi(y), \widehat{g}(x)\rangle \\
    &= \phi^{-1}\left(\lim_{\gamma\to\infty}\argmax_{\mu\in\calM}\left\{\langle\gamma \widehat{g}(x), \mu\rangle - h(\mu)\right\}\right) \\
    &= \phi^{-1}\left(\lim_{\gamma\to\infty}\nabla h^*(\gamma \widehat{g}(x))\right).
\end{align*}
With this form we can compare it to the decoding of our framework which is
\begin{equation*}
    \fhat(x) = \argmin_{z\in\calZ}~\langle\psi(z), \nabla h^*(\ghat(x))\rangle.
\end{equation*}
See \cref{fig:MAP} with explanation.

Finally, we provide the proof of the lower bound on $\zeta_h$ given by \cref{th:lowerboundcalibration}, which is based on the computation of the strong convexity constant w.r.t the Euclidean distance of the negative maximum-entropy potential.
\vspace{3pt}

\begin{proof}[Proof of \cref{prop:calibrationCRFs}]
Recall that the strong convexity constant of a Legendre-type function $h$  w.r.t a norm $\|\cdot\|$ is the inverse of the Lipschitz constant of $\nabla h^*$ w.r.t $\|\cdot\|_{*}$. In this case, $h^*$ corresponds to the partition function $h^*(v) = \log\left(\sum_{y'\in\calY}\exp(\langle v, \phi(y')\rangle\right)$,
and the Hessian corresponds to the Fisher Information matrix which in this case is equal to the covariance $\Sigma(v)$ of $\phi(y)$ under $p_{\phi}(\cdot|v)$, where $p_{\phi}(\cdot|v)=\exp\langle\phi(y),v\rangle/(\sum_{y'\in\calY}\exp\langle\phi(y'),v\rangle)$ is the exponential family with sufficient statistics $\phi$ and parameter vector $v$. Hence, the strong convexity constant of $h$ under the Euclidean norm is the maximal spectral norm of the covariance, which can be upper bounded as
\begin{align*}
   \sup_{v\in\calV} \|\Sigma(v)\|_2
    &= \sup_{v\in\calV}~\left\|\sum_{y\in\calY}q_{\phi}(y|v)\cdot\phi(y)\phi(y)^T\right\|_2 \\ 
    & \leq \sup_{q\in\operatorname{Prob}(\calY)}~\left\|\sum_{y\in\calY}q(y)\cdot\phi(y)\phi(y)^T\right\|_2 \\
    &= \sup_{y\in\calY}~\|\phi(y)\|_2^2. 
\end{align*}
\end{proof}

\section{Binary Classification} \label{app:binaryclassification}
We will present the cost-sensitive case to highlight the fact that a $\phi$-calibrated loss can be calibrated to multiple losses. In this case $\calZ=\calY=\{-1,1\}$ and consider the following cost-sensitive loss $L$ defined as $L(-1,1)=2-c, L(1,-1)=c$ and 0 otherwise with $0<c\leq 1$.
We consider the following embeddings
$\psi(1)=(0,c)^T$, $\psi(-1)=(2-c, 0)^T$, $\phi(1)=(1,0)^T$, $\phi(-1)=(0,1)^T$.
In this case $\calH=\Rspace{2}$, $\calM=\Delta_2$, $r=1$, $k=2$. Hence, the marginal polytope is not full-dimensional. The decoding corresponds to $d(v)=\operatorname{sign}(2q-c)$, where we will abuse notation and set $q=q(Y=1)=\mu_1=\Expect_{Y\sim q}\phi_1(Y)$.

We will focus on surrogate \emph{margin losses} \cite{bartlett2006convexity}, which are losses of the form $S(v, y) = \Phi(yv)$ with $\calV=\Rspace{}$, where $\Phi:\Rspace{}\rightarrow\Rspace{}$ is a non-increasing function with $\Phi(0)=1$. The link function is computed as
\begin{equation}\label{eq:linkmarginloss}
    t(q) = \argmin_{v\in\calV}~\Expect_{Y\sim q}\Phi(Yv).
\end{equation}
Note that it is always the case that the link is symmetric around $q=1/2$, i.e., $t(q)(2q-1)>0$ for $q\neq 1/2$. Hence, in the non-cost-sensitive case ($c=1$), the decoding can be simplified to $d(v) = \text{sign}(2t^{-1}(v)-1) = \text{sign}(v)$. Moreover, the potential function can be computed as
\begin{equation}\label{eq:hmarginloss}
    -h(q) = \min_{v\in\calV}~\Expect_{Y\sim q}\Phi(Yv),
\end{equation}
and it is also symmetric around 1/2.
In the following, we prove that logistic, exponential and square margin losses are $\phi$-calibrated, squared hinge and modified Huber satisfy \cref{eq:linkmarginloss} for injective $t:\Delta_2\rightarrow\calV$ but don't have a BD representation extension to $\calV$, and hinge loss does not satisfy \cref{eq:linkmarginloss} because the corresponding $t$ is not injective.

Here we present those examples and provide the corresponding BD representation (if applicable) and the calibration function using \cref{th:binarycalibration}, which states that $\zeta_h(\varepsilon)=h((1+\varepsilon)/2)-h(1/2)$ when $c=1$.

\begin{remark}[Notation] Throughout this section, we will identify 
    $\Delta_2$ with $[0,1]$ and $\operatorname{affhull}(\Delta_2)$ to $\Rspace{}$ by projecting onto the first coordinate.
\end{remark}

\paragraph{Logistic.} The logistic loss corresponds to $\Phi(u) = \log(1 + \exp(-u))$.
\begin{equation}
\calD=[0,1],\quad
    h(q) = -\operatorname{Ent}(q), \quad t(q)=\log\left(\frac{q}{1-q}\right), \quad \zeta_h(\varepsilon) = 1-\operatorname{Ent}\left(\frac{1+\varepsilon}{2}\right).
\end{equation}
The link is the canonical link and corresponds to $t(q)=h'(q)=\log(q/(1-q))$ with inverse $t^{-1}(v)=(h^*)^{'}(v)=(1+e^{-v})^{-1}$.

\paragraph{Exponential.} The exponential loss corresponds to $ \Phi(u) = \exp(-u)$.
\begin{align*}
    \calD=[0,1],\quad
    \Phi(u) = \exp(-u), \quad h(q) = -2\sqrt{q(1-q)} \\
    t(q) = \frac{1}{2}\log\left(\frac{q}{1-q}\right), \quad \zeta_h(\varepsilon) = 1 - \sqrt{1-\varepsilon^2}.
\end{align*}
The link corresponds to $t(q) = \frac{1}{2}\log(q/(1-q))$ with inverse $t^{-1}(v)=(1+e^{-2v})^{-1}$. It does not correspond to the canonical link, which is
$(h^*)^{'}(v)=\frac{1}{2}\left(1-\frac{u}{\sqrt{4+u^2}}\right)$ and $h'(q)=\frac{1-2q}{\sqrt{q(1-q)}}$.

\paragraph{Square.} The square loss corresponds to $\Phi(u) = (1-u)^2$.
\begin{equation}
       \calD=\Rspace{},\quad
       h(q) = -4q(1-q), \quad t(q) = 2q-1, \quad \zeta_h(\varepsilon) = \varepsilon^2.
    \end{equation}
The link corresponds to $t(q) = 2q-1$ with inverse $t^{-1}(v)=(v+1)/2$. It does correspond to the canonical link up to a multiplicative factor because $(h^*)^{'}(v)=(4+v)/8$ and $h'(p)=4(2q-1)$.

\paragraph{Squared Hinge.} The squared hinge loss corresponds to 
$\Phi(u) = (\max(1-u, 0))^2$.
The link and potential is the same in $\Delta_2$ as the square margin loss. However, in this case the excess bayes surrogate risk reads (see \cite{zhang2004statisticalbehavior})
\begin{equation}\label{eq:excessbayesrisksquaredhinge}
    \delta s(v, q) = (2q-1-v)^2 - q\max(v-1,0)^2 - (1-q)\min(0, v+1)^2.
\end{equation}
We know that for $v\in t(\Delta_2)=[-1,1]$,  $\delta s(v, q)=4(q-t^{-1}(v))^2=D_h(q, t^{-1}(v))$ as the square loss. However this BD representation can't be extended to $\calV=\Rspace{}$ (see \cref{prop:squaredhingenonphical}). 
   
\paragraph{Modified Huber loss.}
The Modified Huber loss \cite{zhang2004statisticalbehavior} corresponds to
\begin{equation}
    \Phi(u) = \left\{
    \begin{array}{ll}
        0 & \text{if } u\geq 1 \\
        -4u & \text{if } u\leq -1\\
        (1-u)^2 & \text{otherwise}
    \end{array}
    \right. .
\end{equation}
The excess Bayes surrogate risk reads (see \cite{zhang2004statisticalbehavior})
\begin{equation}\label{eq:excessbayesriskhuber}
    \delta s(v,q) = (2q-1-T(v))^2 + 2|2q-1-T(v)||q-T(v)|,
\end{equation}
where $T(v)=\min(\max(v, -1), 1)$.
As the squared hinge, it has the same BD representation as the squared margin loss in $t(\Delta_2)=[-1,1]$ but it can't be extended to $\calV=\Rspace{}$ (see \cref{prop:squaredhingenonphical}). 

\paragraph{Hinge. } The hinge loss corresponds to $\Phi(u)=\max(1-u,0)$. We have that
\begin{equation}
   t(q) = \operatorname{sign}(2q-1).
\end{equation}
Note that in this case $t$ is not injective.

\begin{proposition}\label[proposition]{prop:squaredhingenonphical} The Bregman divergence representation of the squared hinge and modified Huber can not be extended to $\Rspace{}$.
\end{proposition}
\begin{proof}[Proof]
We will only do the proof for squared hinge, as the case for modified Huber is analogous.

Observe that $\delta s(v,1)=0$ for $v\geq 1$ and so for any right continuous extension of $t(q)=2q-1$ to $[1,+\infty)$, $\delta s(t(q'),1)=0$ for $q'\geq 1$. In particular, this means that the extension of $h$ must be linear for all $q'\geq 1$. And so 
$\delta s(v, q)$ should be independent of $v\geq 1$ for any $q<1$.
However, this is not the case. Hence, squared hinge does not have a BD representation extension to $\calV=\Rspace{}$.
\end{proof}

Finally, we prove the form of the calibration function for binary margin losses, which can be found at \cite{bartlett2006convexity} for $c=1$ and at \cite{scott2012calibrated} for the asymmetric case.

\begin{proposition}[Binary 0-1 calibration function]\label[proposition]{th:binarycalibration}
The calibration function for a $\phi$-calibrated margin loss can be written as
 $\zeta_h(\varepsilon) = \min_{\alpha\in\{-\varepsilon,+\varepsilon\}}D_h((c+\alpha)/2, c)$. Moreover, if $c=1$, then
the calibration function simplifies to $\zeta_h(\varepsilon) = h\left(\frac{1+\varepsilon}{2}\right)-h\left(\frac{1}{2}\right)$.
\end{proposition}
\begin{proof}[Proof]
    Note that $\langle\psi(1)-\psi(-1),u\rangle = cu_{2}-(2-c)u_{1}$. Hence,
    \begin{equation}
        \calH_{\varepsilon}(1)=\{u\in\Rspace{2}~|~cu_{2}-(2-c)u_{1}\leq \varepsilon\}.
    \end{equation}
    Taking the intersection with $\calM$ gives 
    \begin{equation*}
        \calH_{\varepsilon}(1)\cap\Delta_2=\{q\in[0,1]~|~c-2q\leq \varepsilon\}=\left[\left(\frac{c-\varepsilon}{2}\right), 1\right].
    \end{equation*}
    Analogously, we obtain $\calH_{\varepsilon}(-1)=\{u\in\Rspace{2}~|~(2-c)u_{1}-cu_{2}\leq \varepsilon\}$ and $\calH_{\varepsilon}(-1)\cap\Delta_2=\left[0, \left(\frac{c+\varepsilon}{2}\right)\right]$.
    Recall that $\calD\supseteq[0,1]$.
    We obtain
    \begin{equation}
        D_h(\calH_{\varepsilon}(1)^c\cap\Delta_2, \calH_0(1)\cap\calD) = 
        D_h([0, (c-\varepsilon)/2], [c/2, \infty)\cap\calD) = D_h((c-\varepsilon)/2, c/2),
    \end{equation}
    and 
    \begin{equation}
        D_h(\calH_{\varepsilon}(-1)^c\cap\Delta_2, \calH_0(-1)\cap\calD) = 
        D_h([(c+\varepsilon)/2, 1], \calD\cap(-\infty,c/2]) = D_h((c+\varepsilon)/2, c/2).
    \end{equation}
    Hence, we obtain the desired result:
    \begin{equation}
        \zeta_h(\varepsilon) = \min_{\alpha\in\{-\varepsilon,+\varepsilon\}}D_h((c+\alpha)/2, c/2).
    \end{equation}
    Finally, setting $c=1$ gives $\zeta_h(\varepsilon) = \min_{\alpha\in\{-\varepsilon,+\varepsilon\}}D_h((1+\varepsilon)/2, 1/2)$.
    Note that if $h$ is convex differentiable and symmetric around $1/2$, then $h((1+\varepsilon)/2)=h((1-\varepsilon)/2)$ and $h'(1/2)=0$, which simplifies the expression to
    \begin{equation}
        \zeta_h(\varepsilon) = h\left(\frac{1+\varepsilon}{2}\right)-h\left(\frac{1}{2}\right).
    \end{equation}
\end{proof}

\section{Multiclass Classification}\label{app:multiclassclassification}

In this case, the loss considered is the 0-1 loss with $\calZ=\calY=\{1,\ldots,k\}$.
The 0-1 loss can be written as $L(z,y) = 1 - \langle e_z,e_y\rangle$,
where $e_z,e_y\in\Rspace{k}$ are vectors of the natural basis in $\Rspace{k}$. Setting $\psi(z)=-e_z$, $\phi(y)=e_y$, the marginal polytope $\calM=\Delta_k$ corresponds to the simplex in $\Rspace{k}$, $r=k-1$, which means that the marginal polytope is not full-dimensional. We write $q_j=\mu_j=q(Y_j=1)$. The decoding mapping can be written as $d(v) = \argmax_{j\in[k]}t_j^{-1}(v)$.
See \cref{fig:multiclass_calibration} for a visualization of the calibration sets.

\subsection{One-vs-all Method}

The \emph{one-vs-all method} \cite{zhang2004statistical} corresponds to $S(v,y) = \Phi(v_y) + \sum_{j\neq y}^k\Phi(-v_j)$ with  $\calV=\Rspace{k}$. The surrogate Bayes risk reads $s(v,q) = \sum_{j=1}^k\{q_j\Phi(v_j) + (1-q_j)\Phi(-v_j)\}$.
Note that one can compute $h$ and $t$ as in \eqref{eq:linkmarginloss} and 
\eqref{eq:hmarginloss} independently for each coordinate. Hence, we obtain
\begin{equation}\label{eq:handtonevsall}
    h(q) = \sum_{j=1}^k\bar{h}(q_j) \hspace{1.5cm} t(q) = (\bar{t}(q_j))_{j=1}^k,
\end{equation}
where $\bar{h},\bar{t}$ are the potential and the link corresponding to the associated margin loss. As the individual link $\bar{t}$ is invertible and $\bar{t}(q)(2q-1)>0$ for $q\neq 1/2$, it means that it is increasing and so $t=(\bar{t})_{j=1}^k$ is order preserving. This implies that the decoding can be simplified to $d(v) = \argmax_{j\in[k]}~v_j$.
Note that if the margin loss is $\phi$-calibrated, then the associated one-vs-all method is $\phi$-calibrated for multiclass with 
$h,t$ given by \eqref{eq:handtonevsall} and $\calD=\bar{\calD}^k$, where $\bar{\calD}$ is the (extended) domain of the margin loss. Note that in this case the marginal polytope is always a strict subset of $\calD$: $\Delta_k\subsetneq[0,1]^k\subseteq\bar{\calD}^k$.

We know provide the proof of \cref{th:onevsallcalibration}, which computes the exact calibration function for the one-vs-all method.

\begin{proof}[Proof of \cref{th:onevsallcalibration}]
    Note that by exploiting the symmetries of the problem, one can considerably simplify the problem to 
    \begin{equation}\label{eq:calibrationmulticlass}
        \zeta_h(\varepsilon) = D_h(\{p_2\geq p_1+\varepsilon,p_2\geq p_j\}\cap\Delta_k, \{q_1\geq q_2\}\cap\bar{\calD}^k).
    \end{equation}
    Indeed, as all of the quantities in \cref{eq:exactcalibration} are invariant by permutation, hence, one can get rid of the minimization over $\calZ$ and set $z=1$. Then, also by symmetry, one can reduce to problem to the comparison between $z=1$ and $z=2$.
    
    The idea of the proof is to show that the minimizer of the following minimization problem
    \begin{equation}\label{eq:reductiononesvall}
        \min_{\substack{q_1\geq q_2 \\ p_2\geq p_1 + \varepsilon}} D_h(p,q)
    \end{equation}
    is achieved at the points $q^\star=(1/2,1/2,0,\ldots,0)$ and $p^\star=((1+\varepsilon)/2, (1-\varepsilon)/2, 0, \ldots, 0)$.
    Then, as the constraints in \cref{eq:reductiononesvall} are included in \cref{eq:calibrationmulticlass} and $q^\star\in\calH_0(1)\cap\bar{\calD}^k\subset \{q_1\geq q_2\}$ and 
    $p^\star\in\calH_{\varepsilon}(1)^c\cap\Delta_k\subset \{p_2\geq p_1 + \varepsilon\}$, the result follows.
    
    Note that as $D_h(p,q)=\sum_{j=1}^kD_{\bar{h}}(p_j, q_j)$, we already have $p^\star_j=q^\star_j=0$ for $j=3,\ldots,k$.
    Hence, we reduce the problem to minimizing $D_{\bar{h}}(p_1,q_1) + D_{\bar{h}}(p_2,q_2)$ in $\{q_1\geq q_2\}\cap\{p_2\geq p_1 + \varepsilon \}$.
    Note that the minimum must be necessarily achieved at the boundary. So by setting $\bar{q}=q_1=q_2$ and $\bar{p}=p_1$, one has the following unconstrained problem $D_{\bar{h}}(\bar{p},\bar{q}) + D_{\bar{h}}(\bar{p}+\varepsilon,\bar{q})$.
    Now, using the fact that $\bar{h}$ is symmetric around $1/2$ and its Hessian is non-decreasing in $\bar{\calD}\cap[1/2,\infty)$, we have that $\bar{q}=1/2$ and $\bar{p}=(1-\varepsilon)/2$ is a minimizer. Hence, the result follows.
\end{proof}
\paragraph{Comparison with lower bounds. }
Note that if $\bar{h}$ is $({1}/{\beta_{\|\cdot\|_2}})$-strongly convex in $\bar{\calD}$, then $h$ is $({1}/{\beta_{\|\cdot\|_2}})$-strongly convex in $\calD=\bar{\calD}^k$. Moreover, using that $\max_{z'\neq z}\|e_z-e_{z'}\|_2^2=2$, \cref{th:improvedlowerbound} gives
\begin{equation}
    \zeta_h(\varepsilon)\geq \frac{\varepsilon^2}{4\beta}.
\end{equation}
Note that for square margin loss, where $\bar{h}(q)=-4q(1-q)$ with $\beta_{\|\cdot\|_2}=\frac{1}{8}$, this lower bound is tight.

\subsection{Multinomial Logistic}
\label{app:multinomiallogistic}

Another important example is the \emph{multinomial logistic} surrogate, which corresponds to the loss \eqref{eq:CRFloss} with $\calM=\Delta_k$, which is $S(v,y) = \log(\sum_{j=1}^k\exp(v_j)) - v_y$ where $\calV=\{v\in\Rspace{k}~|~\sum_{j=1}^kv_j=0\}\cong\Rspace{k-1}$. In this case $\calD=\Delta_k$, $t_j^{-1}(v)=\frac{\exp(v_j)}{\sum_{\ell=1}^k\exp(v_{\ell})}$ and so the decoding is also simplified to $d(v)=\argmax_{j\in[k]}~v_j$ by taking the logarithm coordinate-wise, which is a monotone function.
\paragraph{Lower bound on calibration function. }
Note that the entropy is $1$-strongly convex w.r.t the $\|\cdot\|_1$ norm over the simplex. As $\|\cdot\|_{\infty}$ is the associated dual norm and $\max_{z}\|\psi(z)\|_{\infty}=\max_{z}\|e_z\|_{\infty}=1$, \cref{th:lowerboundcalibration} gives
\begin{equation}
    \zeta_h(\varepsilon)\geq \frac{\varepsilon^2}{8}.
\end{equation}

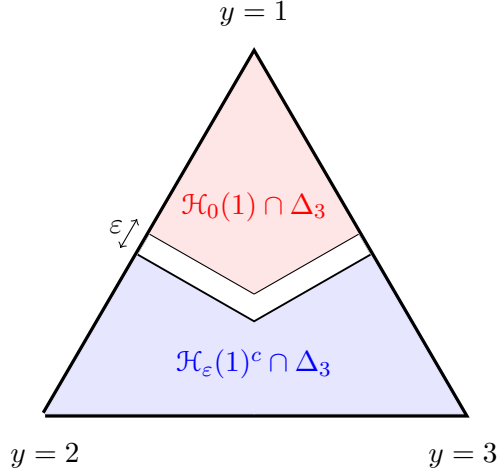
\begin{figure}[t!]
\begin{center}
    \begin{tikzpicture}[domain=0:1, xscale=5.5, yscale=5.5]
    \draw [line width=1mm, color=black] (0,0) -- (1, 0) -- (0.5, 0.866) -- (0, 0); 
    \draw [fill, color=blue!10] (0, 0) -- (0.25, 0.433) -- (0.5, 0.866/3)
    -- (0.5, 0) -- (0,0);
    \draw [fill, color=blue!10] (0.5, 0) --(0.5, 0.866/3) --
    (0.75, 0.433) -- (1, 0) --  (0.5, 0);
    
    \draw[line width=6.5mm, color=black] (0.25, 0.433) -- (0.5, 0.866/3) -- (0.75, 0.433);
    \draw[line width=6mm, color=white] (0.25, 0.433) -- (0.5, 0.866/3) -- (0.75, 0.433);
    \draw [fill, color=red!10] (0.25, 0.433) -- (0.5, 0.866) --
    (0.75, 0.433) -- (0.5, 0.866/3) -- (0.25, 0.433);
    \draw(0.25, 0.433) -- (0.5, 0.866/3);
    \draw(0.75, 0.433) -- (0.5, 0.866/3);
    
    \draw[<->] (0.18, 0.4) -- +  (60:0.08);
    \node at (0.17, 0.45) {$\varepsilon$};
    
    \node at (0, -0.1) {$y=2$};
    \node at (1, -0.1) {$y=3$};
    \node at (0.5, 0.966) {$y=1$};
    
    \node[color=red] at (0.5, 0.5) {$\calH_0(1)\cap\Delta_3$};
    \node[color=blue] at (0.5, 0.12) {$\calH_\varepsilon(1)^c\cap\Delta_3$};
\end{tikzpicture}
    \caption{Illustration of calibration sets for the multiclass with 0-1 loss and $k=3$ labels.}
    \label{fig:multiclass_calibration}
\end{center}
\end{figure}

\section{Multilabel Classification with Hamming Loss} \label{app:multilabelclassification}

 In this case $\calZ=\calY=\{-1,1\}^k$. We have that 
 \begin{equation}\label{eq:hamming}
     L(z,y) = \frac{1}{k}\sum_{j=1}^k1(z_j\neq y_j) = 
    \frac{1}{2} +\sum_{j=1}^k\psi_j(z)\phi_j(y),
 \end{equation}
with $\psi(z)=-z/(2k)$ and $\phi(y)=y$.
In this case $\calM=[-1,1]^k$ which corresponds to the cube.

\subsection{Independent Classifiers}\label{sec:independentclassifier}

The surrogates considered are $S(v,y) = \sum_{j=1}^m\Phi(y_jv_j)$, with $\calV=\Rspace{k}$.
The surrogate Bayes risk reads $s(v,q) = \sum_{j=1}^k\{\frac{\mu_j+1}{2}\Phi(v_j) + \frac{1-\mu_j}{2}\Phi(-v_j)\}$. Note the similarity with the one-vs-all method from multiclass classification. If $\bar{h}$ and $\bar{t}$ are the potential and link for the associated margin loss, we have:
\begin{equation}\label{eq:handthamming}
    h(\mu) = \sum_{j=1}^k\bar{h}\left(\frac{\mu_j+1}{2}\right) \hspace{1.5cm} t(\mu) = \left(\bar{t}\left(\frac{\mu_j+1}{2}\right)\right)_{j=1}^k.
\end{equation}
The decoding simplifies to
$d(v) = (\text{sign}(v_j))_{j=1}^k$.

The calibration function $\zeta_h$ can be computed exactly and it is $k$ times the calibration function of the margin loss: $\zeta_h(\varepsilon) = k\cdot\zeta_{\bar{h}}(\varepsilon)$.
\begin{proposition}[Calibration function for Hamming loss] \label{prop:hammingcalibration}
The calibration function for the Hamming loss is
$$\zeta(\varepsilon) = k\cdot \zeta_{\bar{h}}(\varepsilon).$$
\end{proposition}
\begin{proof}[Proof]
The proof consists of two parts. First, we show that the lower bound $\zeta_h(\varepsilon) \geq k\cdot\zeta_{\bar{h}}(\varepsilon)$ holds, and then we prove that it is actually tight, by showing it is achieved at a pair of points on the minimization problem \eqref{eq:exactcalibration}.

The excess Bayes risk can be written as 
$\delta\ell(z,q)=\frac{1}{k}\sum_{j|z_j\neq z_j(\mu)}|\mu_j|$, where $z(\mu)$ denotes the optimal prediction (see \cref{eq:optimalprediction}). Note that $|\mu_j|$ is the excess Bayes risk of the binary 0-1 loss, hence, $\zeta_{\bar{h}}(|\mu_j|)\leq \delta s_j(\bar{t}(\mu_j),q)$, where $\delta s_j(v_j,q)=\frac{\mu_j+1}{2}\Phi(v_j) + \frac{1-\mu_j}{2}\Phi(-v_j)$ is the excess surrogate Bayes risk of the $j$-th independent classifier.
Hence, 
\begin{align*}
    \zeta_{\bar{h}}(\delta\ell(d(v),q))
    & = \zeta_{\bar{h}}\left(\frac{1}{k}\sum_{j|z_j\neq d_j(v)}|2\bar{t}^{-1}(v_j)-1|\right) &&(\mu_j=2t^{-1}(v_j)-1) \\ 
    &\leq  \frac{1}{k}\sum_{j|z_j\neq d_j(v)}\zeta_{\bar{h}}(|2\bar{t}^{-1}(v_j)-1|) && (\text{Jensen ineq.}) \\
    &\leq  \frac{1}{k}\sum_{j|z_j\neq d_j(v)}\delta s_j(v_j, q) &&(\zeta_{\bar{h}} \text{ calibrates individual classifiers.}) \\
    & \leq\frac{1}{k}\delta s(v, q). &&
\end{align*}
Hence, $\zeta_h(\varepsilon) \geq k\cdot\zeta_{\bar{h}}(\varepsilon)$.
To prove tightness, consider the point $\mu_0=0=(0)_{j=1}^k$ and the point $\mu_{\varepsilon} = (-\varepsilon)_{j=1}^k$ for all $0\leq\varepsilon\leq 1$. If we denote by $\textbf{1}$ the output $(1,\ldots,1)\in\calZ$, we have
that $\mu_0\in\calH_0(\textbf{1})\cap\calM\subset\calH_0(\textbf{1})\cap\calD$ and $\mu_{\varepsilon}\in\calH_{\varepsilon}(\textbf{1})^c\cap\calM$ for all $0\leq\varepsilon\leq 1$ because
\begin{equation}\label{eq:calibrationsetshamming}
    \calH_{\varepsilon}(\textbf{1}) = \left\{u\in\calH~|~ -\frac{1}{k}\sum_{j~|~b_j=1}u_j\leq\varepsilon, \forall b\in\{0,1\}^k
    \right\}.
\end{equation}
Moreover, its Bregman divergence is $D_h(\mu_{\varepsilon}, \mu_0) = \sum_{j=1}^kD_{\bar{h}}(((\mu_{\varepsilon})_j + 1)/2, (1+(\mu_{0})_j)/2) =\sum_{j=1}^kD_{\bar{h}}((1-\varepsilon)/2, 1/2) = k\cdot \zeta_{\bar{h}}(\varepsilon)$. Hence, the lower bound is tight.
\end{proof}


\section{Ordinal Regression} \label{app:ordinalregression}

In this case $\calZ=\calY=\{1,\ldots, k\}$ and these are ordinal output variables instead of categorical. Which means that there is an intrinsic order between them: $1\prec\cdots\prec k$. This is captured by the absolute error loss function defined as
\begin{equation}
    L(z,y) = |z-y|.
\end{equation}
Note that in this case the loss matrix is full-rank, because it is a Toeplitz matrix. Hence, it can be seen as a ``structured'' cost-sensitive multiclass loss.

Let's consider the following embedding $\phi(y)=(2\cdot1(y\geq j)-1)_{j=1}^{k-1}\in\{-1,1\}^{k-1}$ for both $\calZ$ and $\calY$. In this embedding, we have that
\begin{equation}\label{eq:ordinalhamming}
    L(z,y) =  \frac{k-1}{2}-\frac{1}{2}\sum_{j=1}^{k-1}\phi_j(z)\phi_j(y).
\end{equation}
Comparing the expression above with the affine decomposition of the Hamming loss from the section above, we observe that \cref{eq:hamming} and \cref{eq:ordinalhamming} are proportional by a factor $k-1$. 

The decoding from $\calH$ can be written as
\begin{equation}
    z(\mu) = 1 + \sum_{j=1}^{k-1}1(\mu_j\geq 0)\in\calZ.
\end{equation}
The excess Bayes risk is
\begin{equation}
    \delta\ell(z,q) = \sum_{\substack{z(\mu(q))<j\leq z \\ z<j\leq z(\mu(q))}}|\mu_j(q)|.
\end{equation}
By choosing an embedding $\phi$ different than the canonical one used in multiclass classification ($\phi(y)=e_y$), we have performed an affine transformation to the simplex in $k$ dimensions and project it inside the cube $[-1,1]^{k-1}$. What we have gained is that under this transformation the calibration sets have the same structure as for the Hamming loss for $k-1$ labels. Note however, that the marginal polytope for the Hamming loss is the entire cube, while here is a strict subset of the cube. See \cref{fig:ordinalsimplex} for the calibration sets $\calH_0(z)$ for $k=3$ using the canonical embedding.

\begin{figure}[ht!]
    \centering
    \begin{tikzpicture}[domain=0:1, xscale=5, yscale=5]
    \draw [line width=1mm, color=black] (0,0) -- (1, 0) -- (0.5, 0.866) -- (0, 0); 
    \draw [fill, color=blue!10] (0, 0) -- (0.25, 0.433) -- (0.5, 0.866/3)
    -- (0.5, 0) -- (0,0);
    \draw [fill, color=green!10] (0.5, 0) --(0.5, 0.866/3) --
    (0.75, 0.433) -- (1, 0) --  (0.5, 0);

    \draw [fill, color=red!10] (0.25, 0.433) -- (0.5,0) --
    (0.75, 0.433) -- (0.5, 0.866) -- (0.25, 0.433);

    \draw(0.25, 0.433) -- (0.5, 0);
    \draw(0.75, 0.433) -- (0.5, 0);
    
    \node at (0, -0.1) {$y=1$};
    \node at (1, -0.1) {$y=3$};
    \node at (0.5, 0.966) {$y=2$};
    
    \node[color=red] at (0.5, 0.866/2) {${\scriptstyle \calH_0(2)\cap\Delta_3}$};
    \node[color=blue] at (0.25, 0.866/6) {${\scriptstyle \calH_0(1)\cap\Delta_3}$};
    \node[color=green] at (0.75, 0.866/6) {${\scriptstyle \calH_0(3)\cap\Delta_3}$};
\end{tikzpicture}
    \caption{Partition of the simplex corresponding to the absolute loss for ordinal regression with $\calZ=\calY=\{1,2,3\}$.}
    \label{fig:ordinalsimplex}
\end{figure}
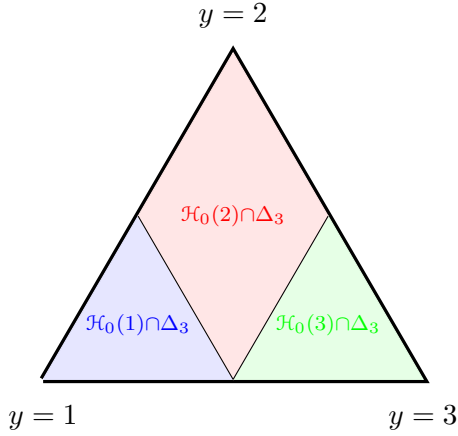

\subsection{All thresholds (AT)} \label{app:AT}

AT methods \cite{lin2006large} correspond to apply an independent classifier (see \cref{sec:independentclassifier}) to the embedding $\phi$. It corresponds to
\begin{equation}
    S(v,y) = \sum_{j=1}^{k-1}\Phi(\phi_j(y)v_j),
\end{equation}
with $\calV=\Rspace{k-1}$. We have that $h(\mu) = \sum_{j=1}^{k-1}\bar{h}((\mu_j+1)/2)$ and $t(\mu) = (\bar{t}((\mu_j+1)/2))_{j=1}^{k-1}$, exactly as for the Hamming loss. Note that in this case however $\calM\subsetneq\calD$ and the decoding mapping is $d(v) = 1 + \sum_{j=1}^{k-1}1(v_j\geq 0)\in\calZ$.

Using the fact that with the embedding $\phi$, the ordinal loss is a $(k-1)$ factor away from the Hamming loss,
and that the marginal polytope is included in the cube, so the minimization is done in a smaller domain, we have that
\begin{equation}\label{eq:calibrationCL}
    \zeta_h(\varepsilon)\geq \zeta_{\text{ham},h}(\varepsilon/(k-1)) = (k-1)\cdot \zeta_{\bar{h}}(\varepsilon/(k-1)),
\end{equation}
where $\zeta_{\text{ham},h}$ is the calibration function of the Hamming loss. With this, we recover the calibration results from \cite{pedregosa2017consistency}.
    
\subsection{Cumulative link (CL)} \label{app:CL}

These methods are of the form \cite{mccullagh1980regression}
\begin{equation}
    S(v, y) = \left\{ \begin{array}{ll}
        -\log(\bar{t}^{-1}(v_1)) & \text{if}~ y=1 \\
        -\log(\bar{t}^{-1}(v_y) - \bar{t}^{-1}(v_{y-1})) & \text{if}~ 1<y<k \\
        -\log(1 - \bar{t}^{-1}(v_{k-1})) & \text{if}~ y=k
    \end{array}\right. ,
\end{equation}
with $\calV=\Rspace{k-1}$.
In this case, it corresponds to consider the decomposition $L(z,y)=\langle L^T\cdot e_z,e_y\rangle$,  (i.e., $\phi(y)=e_y\in\Rspace{k}$), $h(q)=-\operatorname{Ent}(q)$ and $\calD=\Delta_k$ with inverse link,
\begin{equation}
    t^{-1}(v)=\left\{ \begin{array}{ll}
        \bar{t}^{-1}(v_1) & \text{if}~ y=1 \\
        \bar{t}^{-1}(v_y) - \bar{t}^{-1}(v_{y-1}) & \text{if}~ 1<y<k \\
        1 - \bar{t}^{-1}(v_{k-1}) & \text{if}~ y=k
    \end{array}\right. ,
\end{equation}
which is not the canonical one. It is called cumulative link because the link is applied to the cumulative probabilities $\bar{t}^{-1}(v_y) = \sum_{j=1}^yq_j=(\mu_y+1)/2$. The decoding can be written as $d(v) = 1 + \sum_{j=1}^{k-1}1(\bar{t}^{-1}(v_j)\geq 1/2)$.
In the case that $\bar{t}(q)(2q-1)>0$ for $p\neq\frac{1}{2}$, then one can directly write (as for AT) $d(v) = 1 + \sum_{j=1}^{k-1}1(v_j\geq 0)$.

The most common link is the logistic link $\bar{t}^{-1}(v) = 1/(1+e^{-v}))$. With this link CL surrogate is convex (see Lemma 8 in \cite{pedregosa2017consistency}).

In this case, the exact calibration function is not easy to calculate due to the lack of symmetry of the calibration sets (see \cref{fig:ordinalsimplex}). However, it is straightforward to apply the lower bound by using the fact that the entropy is $1$-strongly convex w.r.t the $\|\cdot\|_1$ norm and $c_{\psi,\|\cdot\|_{\infty}} = \max_{z\in\calZ,y\in\calY}|z-y| = k-1$ using the fact that $F_z=L^T\cdot e_z=L_z$.
Hence, applying \cref{th:lowerboundcalibration} we obtain:
\begin{equation}\label{eq:lowerboundCL}
\zeta_h(\varepsilon) \geq \frac{\varepsilon^2}{8(k-1)^2}.
\end{equation}
Note that this lower bound has a factor $(k-1)^{-2}$ instead of the $(k-1)^{-1}$
of \cref{eq:calibrationCL}. 

This explains the experiment of Fig. 1 from \cite{pedregosa2017consistency}, where they show that the calibration function \eqref{eq:calibrationCL} of AT is larger than the calibration function for CL. However, they were not able to provide any result such as \eqref{eq:lowerboundCL}.

\section{Ranking with NDCG Measure} \label{app:ranking}

Let $\calZ=\mathfrak{S}_m$ be the set of permutations of $m$ elements and $\calY=[\bar{R}]^m$ the set of relevance scores for $m$ documents. Let
the \emph{gain} $G:\Rspace{}\rightarrow\Rspace{}$ be an increasing function and the \emph{discount} vector $D=(D_j)_{j=1}^m$ be a coordinate-wise decreasing vector. The NDCG-type losses are defined as the normalized discounted sum of the gain of the relevance scores ordered by the predicted permutation: \begin{equation}\label{eq:ndcgtype-0}
    L(\sigma, r) = 1 - \frac{1}{N(\bar{r})}\sum_{j=1}^m G([\bar{r}]_j)D_{\sigma(j)},
\end{equation}
where $N(\bar{r})=\max_{\sigma\in\mathfrak{S}_m}\sum_{j=1}^{m}G([\bar{r}]_j)D_{\sigma(j)}$ is a normalizer.

Note that looking at \cref{eq:ndcgtype-0}, we have the following affine decomposition \cite{ramaswamy2013convex, nowak2018sharp}:
\begin{equation}
    \psi(\sigma) = -(D_{\sigma(j)})_{j=1}^m, \quad \phi(\bar{r}) = \left(\frac{G([\bar{r}]_j)}{N(\bar{r})}\right)_{j=1}^m.
\end{equation}
Inference from $\calH$ corresponds to $z(\mu) =  \operatorname{argsort}_{\sigma\in\mathfrak{S}_m}(\mu_j)_{j=1}^m$.
If we now consider a strictly convex potential defined in $\Rspace{m}$ and the canonical link, we recover the group of surrogates presented in \cite{ravikumar2011ndcg}. With our framework, Fisher consistency comes for free by construction, we recover the same lower bound on the calibration function of their Thm. 10 from \cref{th:lowerboundcalibration} and the same improvement under low noise of their Thm. 11 from \cref{th:calibrationriskslownoise}.

\section{Graph Matching}\label{app:graphmatching}


In graph matching, the input space $\calX$ encodes features of two graphs $G_1,G_2$ with the same set of nodes, and the goal is to map the nodes from $G_1$ to the nodes of $G_2$. The loss used for graph matching is the Hamming loss between permutations defined as
\begin{equation}
    L(\sigma, \sigma') = \frac{1}{m}\sum_{j=1}^m1(\sigma(j)\neq\sigma'(j)) = 1 - \frac{\langle X_{\sigma},X_{\sigma'}\rangle_{F}}{m} = 1 + \langle\psi(\sigma), \phi(\sigma')\rangle_F,
\end{equation}
where $X_{\sigma}\in\Rspace{m\times m}$ is the permutation matrix associated to the permutation $\sigma$ and the embeddings are $\psi(\sigma)=-X_{\sigma}/m$ and 
$\phi(\sigma)=X_{\sigma}$.
In this case, the Bayes risk reads
\begin{equation}
    \ell(z,q) = 1 - \frac{\langle X_{\sigma}, P(q)\rangle_F}{m},
\end{equation}
where $P(q) = \sum_{\sigma'}q(\sigma')X_{\sigma'}$ and $\calH=\Rspace{k}$ with $k=m^2$. The Bayes optimum is computed through \emph{linear assignment} as
\begin{equation}
    z(P)=\argmax_{\sigma'\in\mathfrak{S}}~\langle X_{\sigma'},P\rangle_F.
\end{equation}
In this case, the marginal polytope corresponds to the polytope of doubly stochastic matrices (also called Birkhoff polytope),
\begin{equation}
    \calM = \{P\in\Rspace{m\times m}~|~P^T\textbf{1}=\textbf{1}, P\textbf{1}=\textbf{1},0\leq P_{ij}\leq 1, 1\leq i,j\leq m\},
\end{equation}
which has dimension $r=\operatorname{dim}(\calM)=k^2-2k+1<k^2$.
One might consider CRFs \cite{petterson2009exponential}, however, the inverse of the canonical link requires performing inference to the associated exponential family (see \cref{app:CRFs}) and this corresponds to computing the permanent matrix which is a \#P-complete problem. 
A possible workaround is to estimate the rows of the matrix $P$ independently with a multiclass classification algorithm and then perform linear assignment with the estimated probabilities. For instance, if one performs multinomial logistic regression independently at each row, it corresponds to the potential 
$h(P)=-\sum_{j=1}^m\operatorname{Ent}(P_j)$ where $P_j$ is the $j$-th row of the matrix $P$ and $\calD$ is the polytope of row-stochastic matrices,
\begin{equation}
    \calD = \{P\in\Rspace{m\times m}~|~P\textbf{1}=\textbf{1},0\leq P_{ij}\leq 1, 1\leq i,j\leq m\} = \prod_{j=1}^m\Delta_m\supsetneq\calM,
\end{equation}
which has dimension $k^2-k<k^2$ strictly larger than the dimension of the marginal polytope.
As the sum of entropies is $1$-strongly convex w.r.t the $\|\cdot\|_1$ norm and $c_{\psi, \|\cdot\|_{\infty}} = \frac{1}{m}$, \cref{th:lowerboundcalibration} gives,
\begin{equation}
    \zeta_h(\varepsilon) \geq \frac{m^2\varepsilon^2}{8}.
\end{equation}

\end{document}